\documentclass[
    letter, 11pt,
    colorlinks,
    linkcolor=blue,
    citecolor=blue,
    pagebackref
    ]{article}

\usepackage[utf8]{inputenc}
\usepackage[margin=1in]{geometry}

\usepackage{amsmath,amsfonts,amsthm,amssymb}
\usepackage{mathtools}

\usepackage[numbers]{natbib} 
\usepackage{thmtools,thm-restate}

\declaretheorem[parent=section]{theorem}
\declaretheorem[sibling=theorem]{lemma}

\declaretheorem[sibling=theorem]{claim}

\newtheorem*{theorem*}{Theorem}

\declaretheoremstyle[
        spaceabove=\topsep, 
        spacebelow=\topsep, 
        bodyfont=\normalfont,
        notefont=\normalfont\bfseries,
        notebraces={}{},
        qed=$\blacksquare$, 
    ]{proofstyle}
\declaretheorem[style=proofstyle,numbered=no,name=Proof]{proof}

\usepackage[extdef=true]{delimset}
\usepackage{hyperref}

\usepackage[capitalise]{cleveref}
\usepackage{crossreftools}

\crefname{claim}{Claim}{Claims}

\newcommand{\reals}{\mathbb{R}}
\newcommand{\W}{\mathcal{W}}
\newcommand{\E}{\mathop{\mathbb{E}}}
\newcommand{\dotp}{\boldsymbol{\cdot}}

\DeclareMathOperator*{\sign}{sign}

\newcommand{\wrgd}{w^{\lambda\textrm{GD}}}
\newcommand{\wgd}{w^{\textrm{GD}}}
\newcommand{\wsgd}{w^{\textrm{SGD}}}

\newcommand{\vbar}{\bar{v}}

\renewcommand{\epsilon}{\varepsilon}
\newcommand{\R}{\mathbb{R}}

\newcommand{\cI}{\mathcal{I}}
\newcommand{\cW}{\mathcal{W}}
\newcommand{\f}[1]{f_{\scriptscriptstyle {(#1)}}}
\newcommand{\F}[1]{F_{\scriptscriptstyle {(#1)}}}
\newcommand{\medcup}{\textstyle \bigcup}
\newcommand{\Z}{\mathcal{Z}}

\newcommand{\discolorlinks}[1]{{\hypersetup{hidelinks}#1}}

\newcommand{\ignore}[1]{}

\makeatletter
\def\blfootnote{\xdef\@thefnmark{}\@footnotetext}
\makeatother

\title{SGD Generalizes Better Than GD \\ (And Regularization Doesn't Help)}
\author{
    Idan Amir\thanks{Department of Electrical Engineering, Tel Aviv University; \texttt{idanamir@mail.tau.ac.il}.} 
    \and Tomer Koren\thanks{School of Computer Science, Tel Aviv University and Google Research, Tel Aviv; \texttt{tkoren@tauex.tau.ac.il}.} 
    \and Roi Livni\thanks{Department of Electrical Engineering, Tel Aviv University; \texttt{rlivni@tauex.tau.ac.il}.} 
}
\date{\today}

\begin{document}
\maketitle

\blfootnote{Accepted for presentation at the Conference on Learning Theory (COLT) 2021}

\begin{abstract} 
We give a new separation result between the generalization performance of stochastic gradient descent (SGD) and of full-batch gradient descent (GD) in the fundamental stochastic convex optimization model. While for SGD it is well-known that $O(1/\epsilon^2)$ iterations suffice for obtaining a solution with $\epsilon$ excess expected risk, we show that with the same number of steps GD may overfit and emit a solution with $\Omega(1)$ generalization error.
Moreover, we show that in fact $\Omega(1/\epsilon^4)$ iterations are necessary for GD to match the generalization performance of SGD, which is also tight due to recent work by \citet{bassily2020stability}.
We further discuss how regularizing the empirical risk minimized by GD essentially does not change the above result, and revisit the concepts of stability,  regularization, implicit bias and the role of the learning algorithm in generalization. 
\end{abstract}




\section{Introduction}

The setting of Stochastic Convex Optimization (SCO) assumes a learner that observes a finite sample of convex functions drawn i.i.d.\ from some unknown distribution and in turn has to provide a parameter that minimizes the expected function with respect to the true and unknown distribution. This is a very simple and clean setting that has received considerable attention in the last two decades 
which culminated in remarkable bounds for 
both the statistical sample complexity as well as the optimization complexity.
%

The two most common and well-known optimization methods in SCO are \emph{Gradient Descent}~(GD) and \emph{Stochastic Gradient Descent}~(SGD). In the first method, one optimizes the empirical risk over a sample by computing iteratively the full-batch gradient; the second method is a ``lighter'' version that samples at each iteration a fresh new example that is used to form an unbiased estimate of the gradient. 
Perhaps surprisingly, even though SGD may seem like a noisy, inaccurate version of GD, it is well known (e.g., \cite{hazan2019introduction}) that the former enjoys an optimal rate and converges after $O(1/\epsilon^2)$ iterations to an $\epsilon$-optimal solution with respect to the \emph{true} underlying distribution, independently of the dimension of the problem. 
This in turn makes it highly suitable for large-scale optimization where the computational costs of the iterations should be taken into account~\citep{bottou2011}. 

Even more surprisingly, while SGD is relatively well understood in terms of its sample complexity bounds,
our understanding of full-batch GD is still lacking;
in fact, it has remained an open question whether GD can obtain the same dimension-independent guarantees attained by SGD.
While this question has been studied under various assumptions such as smoothness and strong convexity~\citep{hardt2016train,bousquet2002stability}, in its general case it has remained largely unresolved.

\paragraph{Our contributions.}

In this work, we give a new separation result between the generalization performance of SGD and of full-batch GD in the context of SCO.
We show that if one runs GD for $O(1/\epsilon^2)$ iterations (with any fixed learning rate) the algorithm may overfit and exhibit a constant gap between empirical error and true error, and in fact, no less than $\Omega(1/\epsilon^4)$ iterations are necessary for it to generalize to within $\epsilon$. 
Interestingly, this last bound turns out to be tight and matches the upper bound implied by a recent stability analysis of GD due to \citet{bassily2020stability}.
In contrast, as discussed above, SGD (with a suitable step size) generalizes after merely $O(1/\epsilon^2)$ steps.
Thus, SGD is not merely a ``light'' noisy version of GD---it is in fact a superior algorithm that enjoys improved generalization guarantees.
To the best of our knowledge, this result is the first to provide such a quantitative separation in generalization performance between these two natural algorithms.

We then proceed to study the role of regularization in optimization. Regularization is known to be a key aspect in SCO: in particular, \citet{shalev2009stochastic} demonstrated that, while an empirical risk minimizer (ERM) might overfit, regularized-ERMs do not. As such, it is natural to ask whether adding regularization to the optimization algorithm improves its performance. 
We show that applying GD to a regularized empirical risk (and choosing the natural learning rates for this setting) would in general require the learner to achieve empirical error $O(\epsilon^2)$ in order to enjoy generalization error $\epsilon$. Overall, then, order of $\Omega(1/\epsilon^4)$ iterations are still necessary for full-batch GD even with added regularization.

Finally, as we further discuss in \cref{sec:discussion} below, our construction allows us to revisit and draw new insights on some of the existing notions and tools in theoretical machine learning such as regularization, stability, implicit bias, and their role towards generalization within the framework of SCO. 


\paragraph{Our techniques.}

The technical heart of our work is a new generalization lower bound for GD that builds upon the two works of \citet{shalev2009stochastic} and \citet{bassily2020stability}. The work of \citet{shalev2009stochastic} was the first to demonstrate that in SCO an empirical risk minimizer may fail to learn. More formally, they showed that even though a regularized ERM can learn with dimension-independent sample complexity, $\Omega(\log d)$ examples are necessary so that any ERM will not overfit; This was later improved by \citet{feldman2016generalization} to $\Omega(d)$.
However, there is still a gap between showing that an abstract ERM can potentially fail and analyzing the performance of concrete algorithm such as GD. In detail, the result of \citet{feldman2016generalization} shows a learning problem where there are some ``bad'' solutions but in contrast there are simpler and easier to find ``good'' solutions and it is not expected that a reasonable algorithm will overfit in that problem (in fact, in its most naive form the initialization at zero is an optimal point). The work of \citet{shalev2009stochastic} demonstrates an example where there is \emph{a unique} minimum. Hence, \emph{any} empirical risk minimizer will fail to learn. However, even this result is limited and cannot be used to rule out the performance of GD. Indeed while the minimum is bad and unique, there still are many approximately good solutions and only at a very high level of accuracy the algorithm starts to fail. In fact, only at an exponentially small training accuracy we obtain a guarantee of overfitting. As such, a reasonable algorithm such as, say, \emph{gradient descent} reaching to $O(1/\sqrt{n})$ optimization accuracy (which is the generalization error to begin with) will not fail. In fact, it will not fail as long as we run it for less than $2^{O(n)}$ iterations.

The second work we rely on is the work of \citet{bassily2020stability} that demonstrated that GD may be an unstable algorithm \citep{bousquet2002stability}---a necessary condition for an algorithm to overfit. Utilizing these two constructions we construct a new example where GD is unstable and converges to one of the ``bad'' ERM solutions. 
We point out that mere instability and lack of uniform convergence are not sufficient for an algorithm to overfit. Indeed, \citet{bassily2020stability} also demonstrated that SGD is unstable (on the same example on which GD is shown to be unstable), but at the same time, SGD comes with provable guarantees and does not overfit. Therefore, constructing such an example, even though utilizes previous constructions, does not follow some generic reduction.

\section{Problem Setup and Background}
\label{sec:setup}

We consider the standard setting of stochastic convex optimization. A learning problem consists of a fixed domain $\W_d \subseteq \reals^d$ in $d$-dimensional Euclidean space, and a loss function $f:\W_d \times \Z \to \reals$, parameterized by a parameter space $\Z$, where for each fixed $z\in \Z$ we assume the function $f(w;z)$ as a function of $w$ is $L$-Lipschitz and convex. 
We will treat throughout $L$ as a constant; in particular, in all our constructions $L$ will be fixed, and will not depend on other parameters of the problem (specifically, $d,\eta, T$ and $n$, as discussed next). We will normally choose $\W_d$ to be the unit-ball in $\reals^d$. If the dimension is fixed, and there is no room for confusion we will also suppress dependence on $d$ and write \[\W=\{w: \|w\|\le 1: w\in \reals^d\}.\]

In this setting, a learner is provided with a sample $S=z_1,\ldots,z_n$ of i.i.d.\ examples drawn from an unknown distribution $D$ and needs to optimize the \emph{true risk} (or \emph{expected risk}, or \emph{true loss}) which we define:
\begin{equation}\label{eq:true} F(w) = \E_{z\sim D}[f(w;z)].\end{equation}
More formally, given the sample $S$ the learner should return, in expectation, a parameter $w_S$ with $\epsilon$-optimal true loss. Namely,
\[\E_{S\sim D^n}[F(w_S)]\le \min_{w^\star\in \W}F(w^\star)+\epsilon.\]
For high probability rates, note that $f$ is Lipschitz, hence bounded in the unit ball and we will mostly care about lower bounds. In turn, bounds in expectation can be readily turned into probability bounds using standard Markov inequality.

We also follow the standard algorithmic assumption in optimization  which assumes the existence of a \emph{first order oracle} for $f$. Namely, for any $w,z \in \W \times \Z$ the learner has access to a procedure that provides her with the value $f(w;z)$ and with the subgradient $\nabla f(w;z)$ with respect to $w$ (\cite{nemirovsky1983problem}; see also \cite{hazan2019introduction,Bubeck15} for a more extensive background).

\paragraph{Stochastic Gradient Descent.}

Perhaps one of the most well studied optimization methods in SCO is \emph{Stochastic Gradient Descent} (SGD). In this method, the algorithm iteratively chooses a parameter $w_t$ (we will always take $w_0=0$) and at step $t$ makes the update
\[\wsgd_{t+1}=\Pi_{\W}\left(\wsgd_t-\eta \nabla f(\wsgd_t,z_{t+1})\right),\quad \wsgd_S:=\frac{1}{T}\sum_{t=1}^T \wsgd_t.\] 
where $\Pi_{\W}$ is the Euclidean projection over the set $\W$. It is well known (see for example \cite{shalev2014understanding,hazan2019introduction}) that if one runs SGD with a learning rate $\eta=\Theta(1/\sqrt{n})$ for $T=n$ iterations then the output $\wsgd_S$ has:

\begin{equation}\label{thm:sgd} \E_{S\sim D^n}[F(\wsgd_S)] \le \min_{w^\star\in \W}F(w^\star)+ O(1/\sqrt{n}).\end{equation}
In particular, for $\epsilon>0$ SGD succeeds to learn to within $\epsilon$-accuracy with an order of $\Omega(1/\epsilon^2)$ calls to a first order oracle, and $\Omega(1/\epsilon^2)$ samples. 

\paragraph{Empirical risk.}

An alternative method to stochastic gradient descent is to optimize the \emph{empirical risk} over a sample $S$, defined next:
\begin{equation}\label{eq:erm} F_S(w)= \frac{1}{n}\sum_{i=1}^n f(w;z_i).\end{equation}
Using standard discretization and covering techniques one can show that if $n=\Omega(d/\epsilon^2)$ then \emph{any} algorithm that optimizes $F_S$ to accuracy $\epsilon$ will also have roughly $O(\epsilon)$ test error (e.g., \cite{shalev2014understanding}). 
In fact, when $n= \Theta(d/\epsilon^2)$ the empirical loss approximates the true loss uniformly, for all $w\in \W$. 
\citet{shalev2009stochastic} showed that a dependence on $d$ in the uniform convergence rate is necessary, and \citet{feldman2016generalization} proved that a linear dependence is in fact tight. 
We emphasize that the dimension dependence is necessary only for uniform convergence; indeed, SGD which does not rely on such arguments, does not exhibit dimension dependencies.

\paragraph{Gradient Descent.}

A concrete way to minimize the empirical risk in \cref{eq:erm} is with (full-batch) \emph{Gradient Descent}. We consider the following update rule \begin{equation}\label{eq:gdrule}\wgd_{t+1}=\Pi_{\W}\left(\wgd_t-\eta \nabla F_S(\wgd_t)\right),~\quad~\wgd_S:=\frac{1}{T}\sum_{t=1}^T\wgd_t .\end{equation}
The output of GD is then given by the averaged sequence, $\wgd_S$.
The optimization error of GD is governed by the following equation for any choice of parameters $\eta$ and $T$
(see for example, \cite{hazan2019introduction, Bubeck15}):
\begin{equation}\label{thm:gd}
    F_S(\wgd_S)
    \le
    \min_{w^\star \in \W} F_S(w^\star) +O\left(\eta + \frac{1}{\eta T}\right)
    .
\end{equation}
In particular, with a choice $\eta = O(\epsilon)$ and $T=O(1/\epsilon^2)$ we can optimize $F_S$ up to accuracy $\epsilon>0$. Using the naive dimension-dependent sample complexity bound, we have that if $n=O(d/\epsilon^2)$, $T=O(1/\epsilon^2)$, $\eta=O(\epsilon)$ we can bound both the optimization error as well as generalization error and achieve true loss of order $\epsilon$.
Recently \citet{bassily2020stability} provided the first, dimension-independent, generalization bound:
\begin{theorem*}[{\citealp[Thm 3.2]{bassily2020stability}}] 
Let $D$ be an unknown distribution over $\Z$ and suppose $f(w;z)$ is $O(1)$-Lipschitz and convex w.r.t.\ $w\in \W$ where $\W$ is the unit ball in $\reals^d$ then running GD over an i.i.d.\ sample $S$ with step size $\eta$ for $T$ rounds yields the following guarantee 
\begin{equation}\label{thm:bassily} 
    \E_{S\sim D^n}[F(\wgd_S)] 
    \le 
    \min_{w^\star\in \W}F(w^\star)+ O\left(\eta\sqrt{T} +  \frac{1}{\eta T}+\frac{\eta T}{n}\right)
    .
\end{equation}
\end{theorem*}
In particular, for $n=1/\epsilon^2$, choosing $\eta=\epsilon^3$ and $T=\Omega(1/\epsilon^4)$ provides:
\[ 
    \E_{S\sim D^n}[F(\wgd_S)] 
    \le 
    \min_{w^\star\in \W}F(w^\star)+ O(\epsilon)
    .
\]
Notice the suboptimality in terms of $\epsilon$. The above bound requires $T=\Omega(1/\epsilon^4)$, which is suboptimal compared with the guarantee of \cref{thm:sgd} for SGD, as well as the dimension dependent generalization bound that requires $T=\Omega(1/\epsilon^2)$. We will show that the above bound is in fact tight. Namely, if $n=O(\log d)$ then for any learning rate we need at least $T=\Omega(1/\epsilon^4)$ iterations for GD to achieve $O(\epsilon)$ true loss.

\paragraph{Regularization.}

It is customary, when minimizing the empirical error, to add a regularization term in order to avoid overfitting. Concretely, given $S=\{z_1,\ldots,z_n\}$, we consider the regularized empirical loss
\begin{equation}\label{eq:reg} F_{\lambda,S}(w) = \frac{\lambda}{2}\|w\|^2 + \frac{1}{n}\sum_{i=1}^n f(w;z_i).\end{equation}
We will consider the following update rule of GD that is known to achieve fast optimal rates for strongly convex functions (in particular, regularized). At step $t$ we take the update rule:
\begin{equation}\label{eq:strongrule} 
    \wrgd_{t+1}
    = 
    \Pi_{\W}\left[\wrgd_t-\eta_{t+1} \nabla F_{\lambda,S}(\wrgd_t)\right]
    , \quad
    \wrgd_S
    :=
    \sum_{t=1}^T \frac{2t}{T(T+1)}\wrgd_t
    ,
\end{equation}
where $\eta_t=\frac{2}{\lambda (t+1)}$.

The above learning rate was suggested by \citet{lacoste2012simpler} where they also demonstrated the optimization guarantee:
\begin{equation}\label{thm:lacoste} F_{\lambda,S}(\wrgd_S) \le \min_{w^\star}F_{\lambda,S}(w^\star) + O\left( \frac{1}{\lambda T}\right).\end{equation}

As for the test error, utilizing \cref{thm:lacoste} one can bound via the empirical error (see \citet[Eq. (24)]{shalev2009stochastic}) (as well as relating the loss of the regularized objective and the non-regularized objective), and achieve the following bound:
\begin{equation}\label{thm:bousquet} 
\E_{S\sim D^n}[F(\wrgd_S)]\le \min_{w^\star \in \W}F(w^\star)+ O\left(\frac{1}{\lambda\sqrt{T}} + \frac{1}{\lambda n}+\lambda\right).
\end{equation}

Similar to before, if we wish to tighten the above bound, we need to choose $\lambda=O(1/\sqrt{n})$ and set $T=O(n^2)=O(1/\epsilon^4)$. We will again show a matching lower bound.

\section{Main Results}\label{sec:results}

We proceed to present our main results which provide accompanying lower bounds to \cref{thm:bassily,thm:bousquet} respectively.

\subsection{Gradient Descent}

The proof of the following result is provided in \cref{prf:main}.
\begin{theorem}\label{thm:main}
Fix $\eta,T$ and $n$. For $d\ge T\cdot 2^{n+5} + 20\cdot \eta^2T^2$, 
 there exists a Lipschitz convex function $f(w;z)$, and a distribution $D$ over $\Z$, such that if $\wgd_S$ is defined as in \cref{eq:gdrule}, then:

\begin{equation}\label{eq:main} \E_{S\sim D^n} [F(\wgd_S)]\ge \min_{w^\star\in \W}F(w^\star)+ 
\Omega\left(\min\left\{ \eta\sqrt{T}+\frac{1}{\eta T},1\right\}\right).\end{equation}
\end{theorem}

Tuning the parameters $\eta$ and $T$, we obtain that for \emph{any} learning rate, to achieve $\epsilon$ true risk we need at least $T=O(1/\epsilon^4)$ iterations. Together with the upper bound in \cref{thm:bassily}, we observe that $T=O(1/\epsilon^4)$, and $n=O(1/\epsilon^2)$ provide optimal rates. 
The main technical novelty of our work is in deriving the first term. Namely, we provide a novel $\Omega(\eta\sqrt{T})$ generalization lower bound. The second term follows from the standard optimization guarantees for GD, which we repeat in the proof. 

\subsection{Regularized Gradient Descent}

We next turn to the question of regularized Gradient Descent. As discussed, it is known that while standard ERM may be liable to overfitting, regularization (and in particular strongly convex regularization) can induce stability which in turn allows learning. We then ask the question if an analogous result appears for algorithmic settings such as GD. Namely, if we optimize over the regularized objective do we guarantee improvement in the performance. Therefore, we now consider the performance of GD on the regularized objective as in \cref{eq:strongrule}. The proof is provided in \cref{prf:strongmain}.

\begin{theorem}\label{thm:strongmain}
Fix $n$, $\lambda>0$ and $T$, and assume $d\ge T\cdot 2^{n+5}\cdot n$. Suppose we run GD over the regularized objective as in \cref{eq:reg} and we output $\wrgd$ as defined in \cref{eq:strongrule}.
Then there exist a distribution $D$ over convex Lipschitz functions, $f(w;z)$ such that:
\begin{equation}\label{eq:strongmain2} \E_{S\sim D^n}\left[F(\wrgd_S)\right]  \ge \min_{w^\star \in \W}F(w^\star)+ \Omega\left(\min\left\{\frac{1}{\lambda \sqrt T} + \ignore{\min\left\{\frac{1}{\lambda n},\frac{1}{\sqrt{n}}\right\}+}\lambda, 1\right\}\right).\end{equation}
In particular, since $\Delta_{\lambda,S} := F_{\lambda,S}(\wrgd_S)-\min_{w} F_{\lambda,S}(w) \le \frac{1}{\lambda T}$ we obtain that
\begin{equation}\label{eq:strongmain1}
\E_{S\sim D^n}\left[ F(\wrgd_S)\right]  \ge \min_{w^\star\in \W} F(w^\star)+ \Omega\left(\min\left\{\sqrt{\frac{\Delta_{\lambda,S}}{\lambda}} \ignore{+ \min\left(\frac{1}{\lambda n},\frac{1}{\sqrt{n}}\right)}+\lambda,1\right\}\right).
\end{equation}
\end{theorem}

Optimizing over the choice of $\lambda$ and $T$, we obtain, again, that at least $T=O(1/\epsilon^4)$ iterations are needed to converge to an $\epsilon$ test error, which is comparable to the guarantee provided for unregularized GD in \cref{thm:bassily}.

\cref{eq:strongmain1} complements the upper bound of \citet[Eq.~24]{shalev2009stochastic}. Taken together, we observe here that $O(\epsilon)$-training error guarantees at best $O(\sqrt{\epsilon})$-test error. Note that in contrast with \cref{thm:bassily} whose last term deteriorates from over-training, under regularization we obtain the reversed effect, and the generalization error stems from \emph{under-training} (see \cref{sec:discussion}, for further discussion). We also mention here the result of \citet{sridharan2008fast} that showed that, in contrast to SCO, in (general) linear models, regularized objectives do enjoy a fast rate and the test error is linear in the train error.

\section{Constructions and Proof Overview}

In this section we give a brief overview over the proof techniques, deferring the complete proofs to \cref{sec:proofs}. 
As discussed above, the main technical contribution of our work is the first term in \cref{eq:main}; namely, in showing that
\begin{equation}\label{eq:tada} F(\wgd_S)- F(w^\star) = \Omega(\eta\sqrt{T}).\end{equation}
The other terms are standard terms that bound the optimization errors of GD. We therefore focus the exposition here on the derivation of \cref{eq:tada}.

The proof relies on two relevant constructions that were presented in \cite{bassily2020stability} and \cite{shalev2009stochastic}: the former provides a lower bound for the stability of GD, while the latter demonstrates a case where uniform convergence fails. Naturally, since both phenomena are necessary to obtain a generalization error, our construction carefully tailors these two ingredients to obtain the final result.

Let us briefly overview the two constructions that we build upon. We begin with the work of \citet{bassily2020stability}.

\paragraph{GD is unstable:} 

To demonstrate instability of GD, \citet{bassily2020stability} constructed the following example that consists of the following two functions:
\begin{equation}\label{eq:r}
v(w)= \gamma v\cdot w, \quad \textrm{and} \quad u(w)= \max\big\{0,\max_{k\in[d]} w(k)\big\},
\end{equation}
where $v=(-1,-1,\ldots,-1)$ and $\gamma$ is an arbitrarily small scalar. Suppose that with some very small probability (order of $1/n$) we observe $v$, and note that the gradient of $v$ slightly perturbs GD from initialization (at zero) towards the positive orthant. The other function we observe is $u$ w.p.~$1-1/n$.

Now to show instability, note that on a typical sample  $v(w)$ will not appear with roughly probability $1/e$. If it is not observed, GD will not move from the origin. On the other hand, if we do observe the function $v(w)$ in the sample, then after the first iteration that perturbs us from zero, all coordinates become positive. At the second iteration, we will observe the gradient $\nabla_2 = \gamma v + e_1$. Taking $\gamma$ to be negligibly small, that means that $w_2\approx w_1 -\eta e_1$, and in turn, since now $w_2(1)=-\eta \le 0$, we have that $\nabla_3=\gamma v+e_2$, etc.\footnote{In our construction, we want to avoid subgradients hence we consider an alternative variant that ensures a well defined gradient at each point. But for the sake of exposition, let us assume that we are provided with the above subgradient oracle.} (Also note that for a sufficiently small $\gamma<1/\sqrt{d}$, the Lipschitz property holds.) 
As such the algorithm will eventually converge to $w_T\approx -\sum_{t=1}^T \eta e_t$. Thus, changing one example leads to a solution that is $\eta \sqrt{T}$ far away and the algorithm is unstable if $\eta=\Omega(1/\sqrt{T})$. One can observe that averaging will not help.

Note though, that the different minima for which the algorithm converges to are all generalizing. In fact all minima are generalizing, hence the example alone is not enough to ensure overfitting.

\paragraph{Uniform convergence fails:} 

The other construction we build upon is by \citet{shalev2009stochastic} which demonstrates that ERM may overfit. Their idea is to consider a distribution over  $z \in \{0,1\}^d$ where each coordinate $z(k)$ is $0$ or $1$ with equal probabilities and a loss function of the form:
\[g(w,z) = \sum_{k=1}^d z(k) w^2(k).\]
The main observation is that if $d$ is large enough then on a sample $\{z_1,\ldots,z_n\}$ of size $n$ (logarithmic in $d$) we expect to see at least one coordinate where $z_i(k)=0$ for all $i$. We will refer to such a coordinate as a \emph{bad} coordinate. Note that for any bad coordinate $k$, the solution $w=e_k$ will achieve zero training error, whereas it has expected loss of $1/2$. Here, however, note that gradient descent will be stable; in particular, the origin is already a minimum. We note that this can be remedied and \citet{shalev2009stochastic} also show how this example can be altered to make sure the bad minimum is unique hence gradient descent will eventually converge to the bad minimum, but their construction will overfit only if we run gradient descent for an exponential number of steps. We, on the other hand, want to show that GD will fail even when it is tuned to achieve, say, $O(1/\sqrt{n})$ error. 

\paragraph{Putting both together:} 

For the sake of exposition, we will show a slightly easier, albeit suboptimal, lower bound of:
\[ F(\wgd_S)-F(w^\star)\ge \eta^2 T,\]
While the above lower bound doesn't match the upper bound of \citet{bassily2020stability}, note that it is still enough to show that gradient descent with step size $\eta=1/\sqrt{T}$ might overfit. 
We next move on to show the above lower bound.

Since we need both overfitting minima as well as instability of GD, then naturally we would like to incorporate both constructions together. The most straightforward idea is consider
\[ \tilde f(w,z) = g(w,z)+\gamma v\cdot w+ u(w)= \sum_{k=1}^d z(k) w(k)^2 +\gamma v\cdot w + \max\big\{0,\max_{k\in[d]} w(k)\big\}.\]

As before, the first step drifts the vector $w$ to the positive orthant. Note, that whenever a bad coordinate is drifted by $u(w)$, we are inflicted a true loss by $g(w,z_i)$. However, if a good coordinate is drifted, the first and last term would counter-act: namely, at the second iteration $\nabla u$ drifts the first coordinate, then the gradient of $\nabla g$ forces (w.h.p.) the first coordinate back to zero unless its a bad coordinate (which will hapen with negligable probability). The construction, thus, fails.

In order to make the above construction work, we need to correlate the bad coordinates with the coordinates that are drifted from zero. This will ensure that the subgradient of $u$ pushes only coordinates on which $g$ is not active. Before we continue with this idea, we would like to point out here that even in this construction, if the first order oracle is allowed to see the whole sample in advance, and choose the subgradient of $u$ adversarially then GD could overfit on this example (In particular, the adversary can choose as subgradient any bad coordinate). In fact, against such a first-order oracle even SGD will fail. This is not allowed though, as the first-order oracle needs to return a subgradient given a single instance $f(w,z)$, without dependence on the sample.

We next move on to discuss how we correlate the bad coordinates in \citet{shalev2009stochastic} with the drift in the construction of  \citet{bassily2020stability}. At each example we draw $z\in \{0,1\}^d$ as in \citet{shalev2009stochastic} (w.p. $1/2$ each coordinate is $0$ or $1$) we then also draw a perturbing vector $v_z$ but now we make sure it perturbs to a positive value only coordinates $k$ on which $z_i(k)=0$, on the other hand for all coordinates $z_i(k)=1$ the vector $v_z$ will in fact have a strong and reverse effect. 

Concretely, letting $\gamma$ be an arbitrarily small scalar we let $v_z(k)=-1$ for $z(k)=0$ and $v_z(k)=n$ for all $k$ such that $z(k)=1$. For this choice of $v_z$ it can be seen that for the average vector $v_S= \frac{1}{n}\sum_{z\in S} v_z,$ we have that $v_S(k)<0$ if and only if $z(k)=0$ for every $z\in S$. Next, our distribution draws at each iteration the function
\[f(w;z) = g(w,z) + \gamma v_z \cdot w + u(w),\]
where again $\gamma$ should be thought of as negligibly small. This leads to the empirical loss:
\[F_S(w)= \frac{1}{n} \sum_{i=1}^n \sum_{k=1}^d z_i(k) w(k)^2 + \gamma v_S\cdot w + \max\big\{0,\max_{k\in[d]} w(k)\big\}.\]
Again, if $d$ is large enough then there are ``many'' bad coordinates on which $z_i(k)=0$ for all $i$. We will need to choose $d$ to ensure that at least $T$ such coordinate exist. Our choice of $v_S$ ensures that at the first iteration, all coordinates where $\sum_{i=1}^n z_i(k)=0$ are perturbed to a positive value. From that iteration on, the term $\max\{w_i,0\}$ will induce over $\wgd_t$ the dynamic depicted in the construction of \citet{bassily2020stability}: $\wgd_{t+1}\approx \wgd_t -\eta e_{i_t}$ where $i_t$ is the $t$'th bad coordinate. Eventually, taking $T$ iterations we obtain a true loss of $\tfrac12 \sum_{t=1}^T \eta^2 = \tfrac12 \eta^2T$.

Since we actually want a loss of $\eta\sqrt{T}$ we need to alter the above construction and we choose a function that behaves more like $\sqrt{g(w,z_i)}$, the gradient of $\sqrt{g}$ is slightly less well behaved and may also cause instability in the ``good'' coordinates, so some consideration need to be taken care of. We refer the reader to \cref{prf:main} for the full proof.

\paragraph{Overview of \cref{thm:strongmain}:}

The proof of \cref{thm:strongmain} exploits roughly the same objective. Certain care need to be taken because of projections. In particular, because over the regularized objective the first iterations take steps of order $O(1/\lambda)$ we necessarily drive out of the unit ball and projections happen -- in distinction from GD without regularization. Again for sake of exposition we will consider the last iterate and prove a weaker bound of $1/(\lambda^2 T)$.

Therefore, for simplicity of the analysis let us start by considering GD over the regularized objective without projections. This algorithm is in fact of interest of its own right and comes with comparable guarantees. Therefore let us consider the update rule
\[ w_{t+1}= w_t - \eta_{t+1}\nabla F_{\lambda,S}(w_t)= (1-\lambda \eta_{t+1})w_t + \eta_t \nabla F_{S}(w_t).\]
where $\eta_{t}=2/(\lambda(t+1))$.
 A simple proof by induction yields the following update rule:
\begin{align*} w_{t+1}& = (1-\frac{2}{t+2})w_t + \frac{2}{\lambda(t+2)}\nabla F_S(w_t)\\
&=\frac{t}{t+2} \frac{2}{\lambda t(t+1)}\sum_{t'=0}^{t-1} (t'+1)\nabla F_S(w_{t'})+ \frac{2(t+1)}{\lambda(t+1)(t+2)}\nabla F_S(w_t) &(\textrm{induction hyp.})\\
&=
\frac{2}{\lambda (t+1)(t+2)}\sum_{t'=0}^{t}(t'+1)\nabla F_S(w_{t'})
.\end{align*}
Therefore, again considering the last iterate we have that

\[F(w)\ge \E[g(w,z)]= \E\left[\sum_{i=1}^d z(i)w(i)^2\right]=\frac{1}{2}\|w_{T}\|^2 =\Theta\left(\frac{1}{\lambda^2 T^4}\sum_{t=0}^T t^2 \right)= \Theta\left(\frac{1}{\lambda^2 T}\right).\]
As before, carefully replacing $g$ with a function that behaves closer to $\sqrt{g}$ leads to the tight bound.

Projections interfere with the above analysis as they contract the vectors and in turn reduce their norm. But if we scale the objective correctly, we can ensure that for enough (say half) of the iterations projections do not occur.

Finally, the analysis above is greatly simplified by the update step suggested in \citet{lacoste2012simpler}. It might seem as if our proof greatly rely on this learning rate. We mention that a similar analysis can be done for learning rate $\eta_t=1/\lambda t$ (and taking the average) as well as for fixed step size $\eta\approx 1/\lambda T$. Nevertheless we leave it as an open problem if there is some first-order optimization method that achieves rate of $T=O(1/\epsilon^2)$ (see \cref{sec:discussion}).

\subsection{The Construction}\label{sec:construction}
We next provide in detail our main construction. We fix $n,d \geq 1$ and parameters $z= (\alpha,\epsilon,\gamma) \in \{0,1\}^d\times \reals^d\times \reals^3$ are such that $0 < \epsilon_1 < \ldots < \epsilon_d$, $\alpha\in \{0,1\}^d$ and $\gamma_1,\gamma_2,\gamma_3 > 0$.
Define a family of convex functions $\f{\discolorlinks{\ref{eq:f_gen}}} : \R^{d} \to \R$ as follows:
\begin{align} \label{eq:f_gen}
    \f{\discolorlinks{\ref{eq:f_gen}}}(w;z)
    =
    \sqrt{\sum_{i\in [d]}\alpha(i) h_\gamma^2(w(i))} +\gamma_1 v_\alpha \dotp w +\gamma_3 r_\epsilon(w),
    \quad\text{with}\quad
    v_{\alpha}(i)
    =
    \begin{cases}
        -\tfrac{1}{2n} & \textrm{if } \alpha(i) = 0;\\
        +1 & \textrm{if } \alpha(i) = 1,
    \end{cases}
\end{align}
where $h_\gamma: \R \to \R$ and $r_\epsilon: \R^d \to \R$ are defined as
\begin{align*}
    h_\gamma(a)
    =
    \begin{cases}
        0 & a \geq -\gamma_2;\\
        a+\gamma_2 & a < -\gamma_2,\\
    \end{cases}
    \quad \textrm{and}\quad
    r_\epsilon(w) 
    = 
    \max\brk[c]1{0,\max_{i \in [d]}\brk[c]{w(i) - \epsilon_i}}
    .
\end{align*}
Observe that $\f{\discolorlinks{\ref{eq:f_gen}}}(w;z)$ are convex, as they are a vector composition of convex functions and since the $\ell_2$-norm is non-decreasing in each argument, (see e.g., \cite[p.~86]{BV2014}). Note also that $\f{\discolorlinks{\ref{eq:f_gen}}}(w;z)$ are $3$-Lipschitz over the Euclidean unit ball for a sufficiently small $\gamma_1 \leq 1 / \sqrt{d}$ and $\gamma_3=1$.

We will consider an $\alpha$ that is distributed uniformly over $\brk[c]{0,1}^d$;
that is, we draw $\alpha \in \brk[c]{0,1}^d$ uniformly at random and pick the
function~$\f{\discolorlinks{\ref{eq:f_gen}}}(w;(\alpha,\epsilon,\gamma))$. The corresponding expected population risk is then
\begin{align*}
    \F{\discolorlinks{\ref{eq:f_gen}}}(w)
    = 
    \E_{\alpha \sim D} \brk[s]{\f{\discolorlinks{\ref{eq:f_gen}}}(w;(\alpha,\epsilon,\gamma))}.
\end{align*}

Now let $S$ be an i.i.d.~sample of size $n$ drawn from this distribution; we
think of $S$ as a multiset of items from $\brk[c]{0,1}^{d}$. Let $F_S$ be the
associated empirical risk: 
\begin{align}\label{eq:f_empirical_gen}
    F_S(w)
    =
    \frac{1}{n} \sum_{\alpha \in S} \f{\discolorlinks{\ref{eq:f_gen}}}(w;(\alpha,\epsilon,\gamma)).
\end{align}
We next provide the key Lemma we will use for the proof of \cref{thm:main} that describes the iteration of GD over \cref{eq:f_empirical_gen} (an analogue results is used in the case of \cref{thm:strongmain}). For the Lemma, given a sample $S$, let us denote by $\cI=\{i: \forall \alpha\in S, \alpha(i)=0\}$ and we will denote $\cI =\{i_1,\ldots, i_K\}$. We will also denote $\vbar= \frac{1}{n}\sum_{\alpha\in S} v_\alpha$.
\begin{lemma} \label{lem:gen}
Let D be a distribution over $z=(\alpha,\epsilon,\gamma)$ where $\alpha\in \{0,1\}^d$ is chosen uniformly and, suppose $\gamma_2 = 2\gamma_1\eta T$, $0<\epsilon_1<\ldots<\epsilon_d<\frac{\gamma_1}{2n}\eta$, $\frac{\gamma_1}{2n}T<1$, $\gamma_1 \leq \frac{1}{2\sqrt{d}\eta T}$ and $\gamma_3=1$ are all chosen deterministically and also that $K\leq \frac{3}{4\eta^2}$. Consider $F_S$ as in \cref{eq:f_empirical_gen}. Then for all $1\leq t \leq \min\brk[c]{T,K}$:
    \[ 
        \nabla F_S(\wgd_t)
        = 
        \gamma_1\vbar + e_{i_t}
    \quad
    \textrm{where,}
    \quad
        \wgd_t
        = 
        -\eta t \cdot \gamma_1 \vbar - \eta\sum_{s=1}^{t-1} e_{i_s}
        .
    \]
    Assuming that $T > K$ then for all $K < t \leq T$:
    \[ 
        \nabla F_S(\wgd_t)
        = 
        \gamma_1\vbar
        \quad
        \textrm{where,}
        \quad
        \wgd_t 
        = 
        -\eta t \cdot \gamma_1 \vbar - \eta\sum_{s=1}^K e_{i_s}
        .
    \]
\end{lemma}
\cref{lem:gen} provides a description of the dynamics of GD over the above loss function with the distribution $D$. One can observe that if the set $\cI$ is ``large'' (which is the set of bad coordinates), then GD converges approximately to a vector $\|\eta \sum_{i\in \cI} e_{i}\| = \eta\sqrt{K}$. As such, if $K=O(T)$, then the algorithm is inflicted loss of $O(\eta\sqrt{T})$, as desired. The proof of \cref{lem:gen} is provided in \cref{prf:gen}.

\section{Discussion}\label{sec:discussion}

In this work we studied the role of the optimization algorithm in learning. We showed that while SGD successfully finds a ``good'' optima that also generalizes, GD minimizes the empirical risk but may suffer large test error. It is not by coincidence that we turned to stochastic convex optimization. Indeed, SCO is perhaps one of few learning models where such a phenomena can exist. Specifically, in setting such as PAC-learning, regression, and general linear models learning follows from uniform convergence. Namely, learnability requires sample complexity that ensures that every minimum of the empirical risk is also an approximate minimum of the true risk. In turn, learning is reduced to empirical risk optimization.\footnote{Of course, even in these simplistic models under distributional assumptions one can emulate phenomena where the algorithm matters. Specifically if we allow to incorporate assumptions that the algorithm is luckily biased towards the right solution then indeed the algorithm matters, but here we try to focus on distribution independent generalization guarantees, and avoid such luckiness-type results.}

In contrast, both in SCO as well as in practice, learning looks much different. In practice, it is a prevalent situation that the learner needs to observe \emph{far less} examples than free parameters, and learning algorithms fully capable of overfitting still succeed to learn \citep{zhang2016understanding,neyshabur2014search}. Also, sometimes perfect-fitting and interpolation induce generalization \citep{belkin2019does, belkin2018overfitting} and in other cases early stopping is the source of generalization \citep{prechelt1998early, cataltepe1999no}. Making the optimization algorithm a key component in the question of generalization.


While under ``luckiness''-type distributional assumptions such phenomena can indeed be recreated even in the most simplistic settings of learning, SCO is a highly attractive theoretical model in this context, and one of few, that exhibits similar phenomena without distributional assumptions, and not less important using the same optimization algorithms as often invoked in practice. As such it is natural to try and study these phenomena in the setting of SCO and to understanding exactly the role of optimization algorithms/regularization/stability as well as perhaps implicit bias and such. We next discuss some of these conclusions as well as future work and open questions:

\paragraph{The (dimension dependent) sample complexity of GD?} 

As discussed, it is well known that given $O(d/\epsilon^2)$ examples, GD (or in fact any ERM algorithm) trained on the dataset will reach $\epsilon$-test error. This work demonstrated that dependence on the dimension is necessary if we are provided with $O(\log d)$ examples.%
\footnote{Here we refer to GD as GD with iteration complexity $O(1/\epsilon^2)$ and learning rate $O(\epsilon)$.} \citet{feldman2016generalization} showed that $\Omega(d)$ examples are necessary so that all ERM algorithms will succeed. This leaves an exponential gap and we leave it as an open question whether GD trained over $\Omega(\log d)$ examples may overfit. In particular, since GD is unstable \citep{bassily2020stability}, and uniform convergence does not apply \citep{feldman2016generalization}, such a result can potentially lead to a new proof technique for generalization. On the other hand, showing that GD overfits even with $O(d)$ examples will also be a significant improvement.

\ignore{
\paragraph{The role of stability in generalization.}

Stability \citep{bousquet2002stability} is a key notion in proving generalization, especially in SCO \citep{hardt2016train, bassily2020stability}. While it is known that in general it is not a necessary condition, its exact role in optimization algorithm is far from understood. Our proof here relies on two constructions that rule out stability as well as uniform convergence. As such, it might seem as if our result follows from some simplistic generic argument that unstable algorithms in SCO fail to learn. 

However, we point out here that \citet{bassily2020stability} also showed that SGD is unstable, and in fact using the same instance example. On the other hand, as we discuss in \cref{thm:sgd}, SGD provably generalizes. Thus, for some miraculous reason, even though all we did was to correlate the instability with the ``bad'' empirical risk minimizers, and even though SGD is also unstable as GD, one of them mysteriously succeeds to learn yet the other fails. A close examination of the dynamics of SGD shows that it can somehow break the correlations we constructed between the perturbing vector $v_S$ and the ``bad'' zero-coordinates. 
}

\paragraph{Early stopping vs.\ perfect fitting.}

As discussed, early stopping and perfect fitting are two (contradictory) important ingredients in the process of optimizing learning algorithms. In \cref{thm:strongmain} we showed that GD, over a strongly convex objective, needs to be trained to $O(\epsilon^2)$ train-accuracy in order to reach $O(\epsilon)$-test accuracy. In contrast note that the upper bound in \cref{thm:bassily} contains a term, $O(\eta T/n)$, that deteriorates due to over training (i.e. $T\to \infty$). Remarkably, both terms rely on stability (i.e. for strongly convex optimization we need high training accuracy for stability, and in the general case, over-training deteriorates the stability).

The last term in \cref{thm:bassily} is the only term for which our main result \cref{thm:main} does not present a matching lower bound, and it is then an open question whether early-stopping is necessary in the setting of stochastic convex optimization. We observe though, that using a construction of \cite{shalev2009stochastic} one can show that some sort of early stopping is indeed necessary. We provide a proof in \cref{prf:overfit}:

\begin{theorem}[informal, see \cref{lem:overfit_main} for exact statement]
\label{thm:overfitl}
For every $\eta,T$ and $n$, for $d\geq T\cdot 2^{n+5}$, there exists a distribution $D$ over convex functions such that
\begin{align*}
    \E_{S\sim D^n}\brk[s]{F(\wgd_S)} - \min_{w\in\cW}F(w)
    \geq 
    \Omega
    \left(1-\frac{2^{2n}}{\eta T}\right)
    .
\end{align*}
\end{theorem}

The ``early stopping'' terms in the above lower bound and in the upper bound presented in \cref{thm:bassily} leaves an exponential gap. It would be interesting to close this gap and to understand the exact effect of early stopping in convex optimization.

\paragraph{The role of regularization.}

We provided here a lower bound that shows that standard algorithms for minimizing regularized objectives don't have any advantage over GD in terms of generalization as long as both are tuned to induce stability.
The lower bound we provide is for a specific choice of learning rate and averaging technique which are common and provide optimal guarantees for minimizing regularized objectives. It is an interesting question whether we can provide a similar lower bound for any choice of dynamic learning rate and averaging technique. 
More broadly, we would like to understand the limitations of first-order optimization methods over the empirical risk. The main take though of the theorem remains, that the lower bound of \cref{eq:strongmain1} is applicable not only to abstract regularized-ERM but in fact to a well--used optimization algorithm with explicit regularization.

\paragraph{The implicit bias of Gradient Descent.} 

One of the most promising tools for understanding generalization in machine learning is the \emph{implicit bias} or \emph{implicit regularization} of optimization algorithms~\citep{neyshabur2014search, gunasekar2018characterizing, gunasekar2018a-implicit, gunasekar2018b-implicit}. This term refers to the algorithms preference towards certain structured solutions which in turn seem to induce generalization.

We would like to revisit this paradigm in the context of SCO. In this work we showed that GD may overfit, but this is in contrast with \citet{bassily2020stability}'s result that with a learning rate $\eta = O(\epsilon^3)$ and $T=O(1/\epsilon^4)$, it succeeds to learn.
Moreover, having seen in \cref{thm:strongmain} that adding regularization is not effective, it is natural, then, to conjecture that the conservative learning rate $\eta = O(\epsilon^3)$ injects implicitly regularization, which in turn accounts for generalization.

However, the work of \citet{dauber2020can} demonstrated that for GD (with any learning rate that yields some optimization guarantee, in particular the above) there is no implicit-bias that accounts for the solution of the algorithm. In other words, no matter what is the learning rate and number of iterations, GD cannot be interpreted as minimizing some regularized version of the original loss function. This result, though, is true only if we don't take the distribution of the data into account and it is an interesting future study to understand if some distribution dependent implicit bias can explain the generalization of GD.
We note, though, that in general, \citet{dauber2020can} did show that there are successful learning algorithms (in fact, SGD) that generalize but their performance does not stem from their bias (even if we take the distribution into account).


\section{Proofs}\label{sec:proofs}
\subsection{Proof of \cref{thm:main}}\label{prf:main}
The proof is an immediate corollary of the following two lower bounds. As one can pick the dominant term between the bounds and thus obtain the desired result.
The first \namecref{thm:gen_main} is the technical heart of our lower bound, and the rest of this section is devoted to prove it.
\begin{theorem}\label{thm:gen_main}
For every $\eta>0$, $T\ge 1$ and $n$, if $d\ge T\cdot2^{n+5}$, then there exists a function $f(w;z):\reals^d\to\reals$ convex and $3$-Lipschitz in $w\in \reals^d$ for every $z\in \Z$, and there exists a distribution $D$ over $\Z$  such that: if $S\sim D^n$ is an i.i.d sample drawn from the distribution $D^n$, then:
\begin{align*}
    \E_{S\sim D^n}\brk[s]{F(\wgd_S)} - \min_{w\in\cW}F(w)
    \geq
    \tfrac{1}{16} \min\brk[c]1{\eta \sqrt{T},\tfrac{1}{3}}
    .
\end{align*}
\end{theorem}

The next lower bound is a well known consequence of the optimization error as well as a standard information-theoretic lower bound. We provide a detailed proof in \cref{prf:opt_1_main}.

\begin{lemma}\label{lem:opt_1_main}
There exists a function $f(w;z): \reals^d \to \reals$ convex and $1$-Lipschitz in $w\in \reals^d$, and a distribution $D$ such that if  $d> 18 \eta^2T^2$ then:
\begin{align*}
    \E_{S\sim D^n}\brk[s]{F(\wgd_S)} - \min_{w\in\cW}F(w)
    \geq
    \tfrac{1}{36}\min{\brk[c]2{\frac{1}{\eta T},9}}
    .
\end{align*}
\end{lemma}

\paragraph{Proof of \cref{thm:gen_main}.} \label{prf:gen_main}

The proof is divided into two parts. The first, and central part, is for $\eta \leq \frac{1}{4\sqrt{3}}$ and the other is for $\eta > \frac{1}{4\sqrt{3}}$.
\paragraph{Case 1 - Assume $\eta\leq \frac{1}{4\sqrt{3}}$:} 
Without loss of generality we assume that
\[  2^n\cdot \max\brk[c]{16,\min\brk[c]{2T,\frac{1}{3\eta^2}}}
    \le 
    d\le \frac{2^n}{2\eta^2}
    .
\]
Indeed, we can assume this as we can always embed the example in any larger dimension.

Next, recall that $\cI=\{i: \forall \alpha\in S, \alpha(i)=0\}$ and we denote $\cI =\{i_1,\ldots, i_K\}$ where $K$ is the cardinality of $\cI$. We start with a probabilistic claim on $K$.
\begin{claim} \label{cla:prob}%
    Suppose that $\log(2\eta^2d) \leq n \leq \min\brk[c]1{\log\brk1{\frac{d}{16}},\log\brk1{\frac{d}{\min\brk[c]{2T,1/(3\eta^2)}}}}$. Then with probability at least $3/4$, it
    holds that $\min\brk[c]{T,1/6\eta^2} \leq K \leq 3/4\eta^2$.
\end{claim}

\begin{proof}
    The probability that a given index $i \in [d]$ is such that $\alpha_i = 0$
    for all $\alpha \in S$ is $2^{-n}$. Thus, the expected number of such
    indices is $\mu = 2^{-n}d$ and the standard deviation is $\sigma =
    \sqrt{2^{-n}\brk{1-2^{-n}}d} \leq \sqrt{\mu}$. By an application of Chebyshev's inequality we
    obtain
    \begin{align*}
        \Pr\brk{K \leq \tfrac12\mu \medcup K \geq \tfrac32\mu)}
        \leq
        \Pr(K \leq \mu - 2\sigma \medcup K \geq \mu + 2\sigma)
        \leq
        \tfrac14
        ~&\textrm{for $\mu \ge 16$}
        .
    \end{align*}
    This gives the claim since $\tfrac12 \mu \geq \min\brk[c]{T,1/(6\eta^2)}$ and $\tfrac32\mu \leq 3/(4\eta^2)$ whenever 
    $\log(2\eta^2d) \leq n \leq \log\brk1{\frac{d}{\min\brk[c]{2T,1/3\eta^2}}}$. Lastly, note that our application of Chebyshev's inequality holds for $\mu \geq 16$, thus we conclude that $n \leq \log\brk{\frac{d}{16}}$.
\end{proof}

Note that the condition of \cref{cla:prob} is satisfied when $16\leq 1/3\eta^2$, which holds for $\eta\leq \tfrac{1}{4\sqrt{3}}$.
We can now lower bound the expected population risk of the GD iterates.
It will be convenient to replace the sequence $w_t$ with the following approximating sequence: define a new sequence, $w'_1, w'_2,\ldots, w'_T$, by setting 
\begin{align*}
    w'_t
    =
    \begin{cases}
        -\eta \sum_{s=1}^{t-1}e_{i_s} & 1\leq t \leq \min\brk[c]{T,K};\\
        -\eta \sum_{s=1}^K e_{i_s} & K < t \leq T.
    \end{cases}
\end{align*}
Denote $w'_S=\frac{1}{T}\sum_{t=1}^Tw'_t$. Using \cref{lem:gen} it is clear that $w'_t-w_t=\eta t \cdot \gamma_1 \vbar$. Now observe that
\begin{align}\label{eq:wgd_diff_norm}
    \norm{w'_t-w_t}
    \leq
    \gamma_1\eta t \norm{\vbar}
    \leq
    \gamma_1\eta t \sqrt{d},
\end{align}
where we have used the fact that $0 \leq \abs{\vbar_i} \leq 1$ for any $i$. In addition, since $\f{\discolorlinks{\ref{eq:f_gen}}}(w)$ is $3$-Lipschitz we have
\begin{align*}
    \F{\discolorlinks{\ref{eq:f_gen}}}(w_S)
    &\ge 
    \F{\discolorlinks{\ref{eq:f_gen}}}(w'_S) - 3\norm{w'_S-w_S} \tag{$3$-Lipschitz}\\
    &\ge
    \F{\discolorlinks{\ref{eq:f_gen}}}(w'_S) - \frac3T \sum_{t=1}^T\norm{w'_t-w_t} \tag{triangle inequality}\\
    &\ge
    \F{\discolorlinks{\ref{eq:f_gen}}}(w'_S) - 3 \gamma_1\eta T \sqrt{d} \tag{\cref{eq:wgd_diff_norm}}
    .
\end{align*}
Note that for any $w\in \cW$
\begin{align*}
    \F{\discolorlinks{\ref{eq:f_gen}}}(w)
    &\geq
    \E_{\alpha\sim D}\brk[s]3{\sqrt{\sum_{i\in [d]}\alpha(i) h_\gamma^2(w(i))}} + \gamma_1\E\brk[s]{v_\alpha} \dotp w + r_\epsilon(w) \tag{$r_\epsilon(w)\geq 0$} \\
    &\geq
    \E_{\alpha\sim D}\brk[s]3{\sqrt{\sum_{i\in [d]}\alpha(i) h_\gamma^2(w(i))}} + \tfrac{1}{2}\gamma_1\brk2{1-\tfrac{1}{2n}} \sum_{i\in [d]} w(i) \tag{ $\E\brk[s]{v_\alpha(i)}=\frac{1}{2}\brk2{1-\frac{1}{2n}}$} \\
    &\geq
    \E_{\alpha\sim D}\brk[s]3{\sqrt{\sum_{i\in [d]}\alpha(i) h_\gamma^2(w(i))}} - \tfrac{1}{2}\gamma_1\sqrt{d} \tag{$\sum_{i=1}^d w(i) \geq -\sqrt{d}$}
    ,
\end{align*}
where in the last inequality we used that $\sum_{i=1}^d w(i) \geq -\norm{w}_1 \geq -\sqrt{d}\norm{w}_2 \geq -\sqrt{d}$ for $w\in\cW$. 
Putting both observations together this implies
\begin{align*}
    \F{\discolorlinks{\ref{eq:f_gen}}}(w_S)
    \geq
    \E_{\alpha\sim D}\brk[s]3{\sqrt{\sum_{i\in [d]}\alpha(i) h_\gamma^2(w'_S(i))}} - \tfrac{1}{2}\gamma_1\sqrt{d} - 3 \gamma_1\eta T \sqrt{d}
    .
\end{align*}
Applying the reverse triangle inequality we also have the inequality:
\begin{align*}
    \sqrt{\sum_{i\in [d]}\alpha(i) h_\gamma^2(w'_S(i))}
    \geq
    \sqrt{\sum_{i\in [d]}\alpha(i) \brk{w'_S(i)}^2} - \sqrt{\sum_{i\in [d]}\alpha(i)\brk{ h_\gamma(w'_S(i)) - w'_S(i)}^2}
    .
\end{align*}
Next, observe that $\abs{h_\gamma(w'_S(i)) - w'_S(i)} \leq \gamma_2=2\gamma_1\eta T$ since $w'_S(i)\leq 0$ for any $i\in[d]$, thus:
\begin{align*}
    \F{\discolorlinks{\ref{eq:f_gen}}}(w_S)
    &\geq
    \E_{\alpha\sim D}\brk[s]3{\sqrt{\sum_{i\in [d]}\alpha(i) \brk{w'_S(i)}^2}} - \tfrac{1}{2}\gamma_1\sqrt{d} - 5 \gamma_1\eta T \sqrt{d} \\
    &\geq
    \frac{\norm{w'_S}}{2} - \tfrac{1}{2}\gamma_1\sqrt{d} - 5 \gamma_1\eta T \sqrt{d} \tag{Jensen's inequality with $\E\brk[s]{\alpha(i)}=\tfrac12$}
    ,
\end{align*}
where a simple observation of $\alpha(i)=\alpha^2(i)$ ensures that the first term is convex. From the definition of $w'_t$ it is clear that for any $t_0<\min\brk[c]{K,T}$ it holds that $w'_t(i_s)=-\eta$ for $s<t_0$ and $t > t_0$. Therefore, setting $t_0=\frac12\min\brk[c]{K,T}$ we have the following inequality
\begin{align*}
    \forall s < \frac12\min\brk[c]{K,T} :
    \qquad
    w'_t(i_s)
    \leq
    \begin{cases}
        - \eta & \tfrac12 \min\brk[c]{T,K} < t \leq T; \\
        0 & o.w.\\
    \end{cases}
\end{align*}
Therefore, the average iterate holds $w'_S(i_s)\leq -\tfrac12\eta$ for any $s<\frac12\min\brk[c]{K,T}$. With this in hand, we can conclude:
\begin{align*}
    \F{\discolorlinks{\ref{eq:f_gen}}}(w_S)
    &\geq
    \tfrac{1}{2\sqrt{2}}\eta\sqrt{\min\brk[c]{K,T}} - \tfrac{1}{2}\gamma_1\sqrt{d} - 5 \gamma_1\eta T \sqrt{d}  \tag{$\norm{w'_S} \geq \tfrac12\eta \sqrt{\tfrac12\min\brk[c]{K,T}}$} \\
    &\geq
    \tfrac{1}{2\sqrt{2}}\eta\sqrt{\min\brk[c]{\tfrac{1}{6\eta^2},T}} - \tfrac{1}{2}\gamma_1\sqrt{d} - 5 \gamma_1\eta T \sqrt{d} \tag{using \cref{cla:prob}} \\
    &\geq
    \tfrac{1}{4}\min\brk[c]1{\eta \sqrt{T}, \tfrac13} - \tfrac{1}{2}\gamma_1\sqrt{d} - 5 \gamma_1\eta T \sqrt{d} \tag{$2\sqrt{2} < 4$ and $\sqrt{6}<3$}
    .
\end{align*}
As the above inequality holds with probability at least 3/4 (\cref{cla:prob}), taking the expectation into account and the fact that in any case $\F{\discolorlinks{\ref{eq:f_gen}}}(w) \geq - \tfrac{1}{2}\gamma_1\sqrt{d}$ and that $\min_{w\in \cW}\F{\discolorlinks{\ref{eq:f_gen}}}(w)\leq\F{\discolorlinks{\ref{eq:f_gen}}}(0)=0$, we attain that
\begin{align*}
    \E_{S\sim D^n}\brk[s]{\F{\discolorlinks{\ref{eq:f_gen}}}(w_S)} - \min_{w\in \cW}\F{\discolorlinks{\ref{eq:f_gen}}}(w)
    &\geq
    \tfrac{3}{16}\min\brk[c]1{\eta \sqrt{T}, \tfrac13} - \gamma_1\sqrt{d} - 5 \gamma_1\eta T \sqrt{d} \\
    &\geq
    \tfrac{1}{8}\min\brk[c]1{\eta \sqrt{T}, \tfrac13}
    .
\end{align*}
For a sufficiently small $\gamma_1$ such that $\gamma_1\brk{1+5\eta T}\sqrt{d}\leq \tfrac{1}{16} \min\brk[c]1{\eta \sqrt{T}, \tfrac13}$.

\paragraph{Case 2 - Assume $\eta > \frac{1}{4\sqrt{3}}$:}
To conclude the proof we are left to show a constant lower bound for the case of $\eta > \frac{1}{4\sqrt{3}}$.
For that matter we define the deterministic convex $2$-Lipschitz function $\f{\discolorlinks{\ref{eq:f_opt_2}}}: \reals \to \reals$:
\begin{align} \label{eq:f_opt_2}
    \f{\discolorlinks{\ref{eq:f_opt_2}}}(w)
    =
    \begin{cases}
        \abs{w-\tfrac14\eta} & \eta\leq 1;\\
        2\abs{w-\tfrac23} & \eta > 1.\\
    \end{cases}
\end{align}
For the case $\eta \leq 1$ the gradients are given by $\nabla\f{\discolorlinks{\ref{eq:f_opt_2}}}(w)=\sign(w-\tfrac16\eta)$. The first GD iterate is then $w_1=\eta$. Observe that $w_1-\tfrac14\eta > 0$. Therefore, the second GD iterate is $w_2=w_1-\eta=0$ and we can deduce that $w_t=\eta$ for odd $t$ and $w_t=0$ for even $t$. This implies that the average iterate holds $w_S\geq \tfrac12\eta$ and we conclude that $\f{\discolorlinks{\ref{eq:f_opt_2}}}(w_S) \geq \tfrac12\eta - \tfrac14\eta \geq \tfrac14\eta$.
For the case $\eta>1$, note that $\nabla \f{\discolorlinks{\ref{eq:f_opt_2}}}(0)=-2$ and therefore the first GD iterate after projection is then $w_1=1$. For the next iterate, observe that $\nabla \f{\discolorlinks{\ref{eq:f_opt_2}}}(1)=2$, which implies that $w_2=-1$. Examine the third iterate, since $\nabla \f{\discolorlinks{\ref{eq:f_opt_2}}}(-1)=-2$ we obtain that $w_3=1$. This entails that $w_t=1$ for odd $t$ and $w_t=-1$ for even $t$. For even $T$ we get that the average iterate is $w_S=0$ and therefore $\f{\discolorlinks{\ref{eq:f_opt_2}}}(w_S)=\tfrac43$. On the other hand, when $T$ is odd we get that the average iterate is $w_S=\tfrac1T$ and therefore $\f{\discolorlinks{\ref{eq:f_opt_2}}}(w_S)=2\abs{\tfrac1T-\tfrac23}\geq \tfrac13$. Putting together both results (for $\eta >1$ and $\eta \leq 1$) we obtain that for $\eta > \frac{1}{4\sqrt{3}}$
\begin{align*}
    \f{\discolorlinks{\ref{eq:f_opt_2}}}(w_S) - \min_{w\in \cW} \f{\discolorlinks{\ref{eq:f_opt_2}}}(w)
    \geq 
    \tfrac14\min\brk[c]{\eta,\tfrac43}
    \geq
    \tfrac{1}{32}
    .
\end{align*}
Note that this result is independent on the dimension as we can simply embed the function $\f{\discolorlinks{\ref{eq:f_opt_2}}}$ in the first coordinate of any large space. Namely, $f(w)=\f{\discolorlinks{\ref{eq:f_opt_2}}}(w(1))$ for $w\in \reals^d$.

\subsection{Proof of \cref{thm:strongmain}} \label{prf:strongmain}
As before, the proof follows from the following two claim which divides the lower bound into two terms. The first term is dealt in \cref{lem:strongcore} and provides the main novelty of this section, and the rest of the section is devoted for its proof. \cref{cl:lambda} accompanies the first result with the standard lower bounds that stems from optimization and sample complexity lower bound:
\begin{theorem}\label{lem:strongcore}
Fix $n$, $\lambda< 3$, $T\ge 3$ and assume $d\ge T\cdot 2^{n+5}$. Suppose we run GD over the regularized objective as in \cref{eq:reg} with learning rate $\eta_t=2/(\lambda(t+1))$ and setting $w_S=\sum_{t=1}^T\frac{2t}{T(T+1)}w_t$. Then, there exists an $f(w,z)$ $3$-Lipschitz and convex over $w\in \W$ and a distribution $D$ supported on $z$ such that: 
\[ \E_{S\sim D^n} \brk[s]{F(\wrgd_S)} - F(w^\star) \ge \frac{3}{4}\min\left\{\frac{1}{8\lambda \sqrt{T+1}}, \frac{1}{16}\right\}.\]
\end{theorem}

The proof is now an immediate corollary of the following standard lower bound of $\Omega\left(\lambda +\min\left\{\frac{1}{\lambda n},\frac{1}{\sqrt{n}}\right\}\right)$, we refer the reader to \cref{prf:lambda} for complete proofs:

\begin{lemma}\label{cl:lambda}
For $\lambda >0$ and fixed $n$, there exists a function  $f(w,z):\reals^d\to \reals$ and a distribution $D$, such that if we run GD over $F_{\lambda,S}$ as in \cref{eq:reg}, with $\eta_t=2/(\lambda (t+1))$ and set $w_S=\sum_{t=1}^T \frac{t}{T(T+1)} w_t$ then
\[\E_{S\sim D^n}[F(\wrgd_S)]-\min_{w^\star\in \W} F(w^\star)\ge 
\frac{1}{128}\min\left\{\min\left\{\frac{4}{\lambda n},\frac{1}{\sqrt{n}}\right\}+64\lambda,128\right\}\]
\end{lemma}

\paragraph{Proof of \cref{lem:strongcore}.}\label{prf:strongcore}
We set $D$ to be a distribution over $z=(\alpha,\epsilon,\gamma)$, similarly to \cref{eq:f_opt_1}, and again we define 
\begin{align} \label{eq:f_genstrong}
    \f{\discolorlinks{\ref{eq:f_genstrong}}}(w;z)
    =
    \sqrt{\sum_{i\in [d]}\alpha(i) h_\gamma^2(w(i))} +\gamma_1 v_\alpha \dotp w +  \gamma_3\cdot r_\epsilon(w),
    \quad\text{with}\quad
    v_{\alpha}(i)
    =
    \begin{cases}
        -\tfrac{1}{2n} & \alpha(i) = 0;\\
        +1 & \alpha(i) = 1,
    \end{cases}
\end{align}
where $\alpha\in \{0,1\}^d$ such that $\alpha(i)=0$ w.p $1/2$, $\epsilon_1< \epsilon_2<\ldots <\epsilon_d< \frac{\gamma_1}{6n(T+1)}$, and $\gamma_3= \min\left\{\frac{\lambda}{2}\sqrt{T-2},1\right\}$. Finally we choose:
\begin{equation}\label{eq:gam12strong}\gamma_2\le \frac{10^{-3}}{\sqrt{T}}\cdot \frac{\gamma_3}{4\sqrt{2}\lambda\sqrt{T+1}} \quad  \gamma_1 \le \min\left\{\frac{10^{-3}}{\sqrt{d}(3+\lambda)}\cdot \frac{\gamma_3}{4\sqrt{2}\lambda\sqrt{T+1}}, \frac{\gamma_2}{\sum_{t=1}^T\eta _t}, \left(\frac{\lambda}{3}\right)^{T+1} \frac{\gamma_3}{(T+1)}\right\} .\end{equation}

Observe that with these choice of parameters, $\f{\discolorlinks{\ref{eq:f_genstrong}}}$ is $3$-Lipschitz. Because we only deal with regularized objectives, throughout this section we suppress dependence in the algorithm and write $w$ instead of $\wrgd$.
Next, as in \cref{eq:reg}, we consider the regularized objective

\[F_{\lambda,S}(w)= \frac{\lambda}{2}\|w\|^2 + \frac{1}{m}\sum_{i=1}^m \f{\discolorlinks{\ref{eq:f_genstrong}}}(w;z_i),\]
The update rule is then given by:
\[w_{t+1}=\Pi_{\W}[w_t-\eta_{t+1}\nabla F_{\lambda,S}(w_t)]= \Pi_{\W}[(1-\lambda \eta_{t+1})w_t-\eta_{t+1}\nabla F_S(w_t)].\]
Next, let us set $\cI=\{j: \forall z_i\in S,~ \alpha_i(j)=0\}$ and we will denote the element of $\cI$ as $i_1\le i_2\le \ldots i_K$. By our assumption on $d$ we have that $n\le \min\brk[c]1{\frac{\log d}{16},\frac{\log d}{2T}}$ then, as in \cref{cla:prob}, we have that $K\ge T$ with probability $3/4$. We will show that if this event happens then:
\[ F(w_S) - F(w^\star) \ge \min\left\{\frac{1}{8\lambda \sqrt{T+1}}, \frac{1}{16}\right\}.\]
The result in expectation then follows.

We first utilize \cref{lem:core} as before. Specifically, we want to show the following claim:
\begin{claim}\label{cl:gardstrong}
For every $t\ge 2$:
\begin{equation}\label{eq:gardstrong} \nabla F_S(w_{t})= \gamma_1 \vbar +\gamma_3e_{i_{t}}.\end{equation}
\end{claim}
The proof of \cref{cl:gardstrong} is left to the end of this section and is provided in \cref{prf:gardstrong} and we continue with the proof of \cref{lem:strongcore}. It will be convenient to replace the sequence $w_{t}$ with the following approximating sequence: define a new sequence,$w'_1, w'_2,\ldots, w'_T$, by setting $w'_1=0$ and for every $t\ge 2$
\begin{equation}\label{eq:surrogatestrong} w'_{t+1}= \Pi_{\W}\left[(1-\eta_{t+1}\lambda)w'_t - \eta_{t+1}  \gamma_3\cdot e_{i_t}\right],\end{equation}
and we set $w'_S=\sum_{t=1}^T \frac{2t}{T(T+1)}\cdot w'_t$. We next claim that 
\begin{equation}\label{eq:stability}
\|w_S-w'_S\| \le \frac{\gamma_1\sqrt{d}}{\lambda}.
\end{equation}
We can prove the above by induction. We show that $\|w_t-w'_t\|\le \frac{\gamma_1\sqrt{d}}{\lambda}$, and then the result holds also for the averaged $w_S$.
For $t=1$ this is immediate from \cref{lem:core} and the following calculation: \[\|w_1-w_1'\|=\|w_1\|=\|\eta_1\nabla F_S(0)\|\le \frac{\gamma_1\sqrt{d}}{\lambda}.\]
Next we assume the statement holds for $t$ and prove for $t+1$:
\begin{align*}
    \|w'_{t+1}-w_{t+1}\| &= \left\|\Pi_{\W}\left[(1-\eta_{t+1}\cdot\lambda)w'_t - \eta_{t+1}  \gamma_3\cdot e_{i_t}\right]- \Pi_{\W}\left[(1-\eta_{t+1}\cdot\lambda)w_t-\eta_{t+1}\gamma_1\vbar - \eta_{t+1}  \gamma_3\cdot e_{i_t}\right]\right\|\\
    &\le 
    \|(1-\eta_{t+1}\cdot\lambda)w'_t - \eta_{t+1}  \gamma_3\cdot e_{i_t}- (1-\eta_{t+1}\cdot\lambda)w_t+\eta_{t+1}\gamma_1\vbar + \eta_{t+1}  \gamma_3\cdot e_{i_t}\|\\
    &=
    \|(1-\eta_{t+1}\cdot\lambda)(w'_t-w_t)\|+\|\eta_{t+1}\gamma_1\vbar\|\\
    &\le \frac{t}{t+2}\|w'_t-w_t\|+\frac{2}{t+2}\frac{\gamma_1\sqrt{d}}{\lambda }\\
    &\le \frac{t}{t+2}\cdot\frac{\gamma_1\sqrt{d}}{\lambda }+\frac{2}{t+2}\frac{\gamma_1\sqrt{d}}{\lambda}\\
    &=  \frac{\gamma_1\sqrt{d}}{\lambda}.
\end{align*}
This establishes \cref{eq:stability}. Next, we formalize the final claim that we will need:
\begin{claim}\label{cl:cl4}
For $t_0\ge T/2$, we have that
\[ w'_{t_0} =\frac{2}{\lambda t_0(t_0+1)}\left( \frac{\lambda T/2(T/2+1)}{2}w'_{T/2}-\sum_{t=T/2+1}^{t_0} \gamma_3 t e_{i_t}\right).\]
and in particular for any $T/2< t_0\le  3T/4$
\begin{equation}\label{eq:ws}
w'_S(i_{t_0}) =- \gamma_3\sum_{t=t_0}^T \frac{2t}{T(T+1)}\frac{2t_0}{\lambda t(t+1)}
\le
- \gamma_3\sum_{t=t_0}^{T}\frac{2}{\lambda T(T+1)}\le
- \frac{\gamma_3}{2\lambda (T+1)}
\end{equation}
\end{claim}
The proof of \cref{cl:cl4} is again deferred to the end at \cref{prf:cl4}, and we proceed with the proof.

We will refer to $\f{\discolorlinks{\ref{eq:f_genstrong}}}(w;z)$ as $f(w;z)$ for the rest of the proof. First, we derive a generic lower bound for the generalization error. 
Since $\E_{z\sim D}[f(w^\star;z)]\le \max_{z} f(0;z)=0$, we have that
\begin{align*}
&\E_{z\sim D}[f(w_S;z)]-\E_{z\sim D}[f(w^\star;z)]
\\
&= \E_{z\sim D}[f(w_S;z)]-\E_{z\sim D}[f(w'_S;z)]\\&\phantom{=} +\E_{z\sim D}[f(w'_S;z)]-\E_{z\sim D}[f(w^\star;z)]\\
&\ge \E_{z\sim D}[f(w'_S;z)]-\E_{z\sim D}[f(w^\star;z)] - 3\frac{\gamma_1\sqrt{d}}{\lambda}  & \textrm{(3-Lipschitzness~of~f) ~\&~\cref{eq:stability}}\\
&\ge \E_{z\sim D}[f(w'_S;z)] - 3\frac{\gamma_1\sqrt{d}}{\lambda}& \E[f(w^\star,z)]\le 0\\ 
&\ge \E_{\alpha\sim D}\brk[s]*{\sqrt{ \sum_{j=1}^T \alpha^2(i_j)h_\gamma^2(w'_S(i_j))}}- \left(\frac{3}{\lambda}+1\right)\sqrt{d}\gamma_1 & \abs{v_{\alpha}\cdot w'_S}\le \|v_{\alpha}\|\le \sqrt{d}\\
&\ge \frac{1}{2}\sqrt{\sum_{j=1}^T h_\gamma^2(w'_S(i_j))} - \left(\frac{3}{\lambda}+1\right)\sqrt{d}\gamma_1 &\textrm{convexity ~of ~norm}\\
\end{align*}
Next, note that for every $i$ and $t$ we have that $w'_t(i)\le 0$. In particular we have that $|h_\gamma(w'_S(i))-w'_S(i))|\le \gamma_2$, hence:
\begin{align*}
\frac{1}{2}\sqrt{\sum_{j=1}^T h_\gamma^2(w'_S(i_j))} 
&\ge \frac{1}{2}\sqrt{\sum_{j=1}^T w'^2_S(i_j)} - \frac{1}{2}\sqrt{\sum_{j=1}^T (w'_S(i_j)-h_\gamma(w_S(i_j))^2}   &\|v\|\ge \|u\|-\|u-v\|\\
&\ge \frac{1}{2}\sqrt{\sum_{j=1}^T w'^2_S(i_j)}-  \sqrt{T}\cdot \gamma_2& |h_\gamma(w'_S(i))-w'_S(i))|\le \gamma_2\\
&\ge
\frac{1}{2}\sqrt{\sum_{j=T/2+1}^{3T/4} w'^2_S(i_j)}-  \sqrt{T}\cdot \gamma_2\\
&\ge  
\sqrt{\sum_{j=T/2}^{3T/4} \left(\frac{\gamma_3}{4\lambda(T+1)}\right)^2}-  \sqrt{T}\cdot \gamma_2&\cref{eq:ws}\\
&\ge
\frac{\gamma_3}{4\lambda\sqrt{2(T+1)}}-  \sqrt{T}\cdot \gamma_2\\ 
\end{align*}
Taken together, and with our choice of $\gamma_1,\gamma_2$ in \cref{eq:gam12strong} we obtain the desired result.

\paragraph{Proof of \cref{cl:gardstrong}.}\label{prf:gardstrong}
We, again, prove the statement by induction and we show that for $w_t$ we have that

\begin{equation}\label{eq:residue} w_t = \sum_{i\notin \cI} \rho^{(t)}_i e_i + \mu^{(t)}\sum_{i> i_{t-1}, i\in \cI} e_i + \sum_{i\le i_{t-1},i\in \cI} \xi^{(t)}_i e_i,\end{equation}
where
\begin{itemize}
    \item $-\gamma_2\le -\gamma_1 \sum_{i=1}^{t-1} \eta_t \le \rho^{(t)}_i\le 0$
    \item $\epsilon_d< \mu^{(t)} \le \frac{\gamma_1}{2\lambda n}$
    \item $\xi^{(t)}_i \le 0$
\end{itemize}
We now assume that the above holds for $w_t$ and prove the statement for $w_{t+1}$. First, by \cref{lem:core} we have that for every $t'\le t$:
\[ \nabla F_S(w_{t'})= \gamma_1 \vbar +\gamma_3\cdot e_{i_{t'}}.\]
Let us denote
\begin{align*}w_{t+1/2}&=\sum_{i\notin \cI} \underbrace{\left((1-\lambda \eta_t)\rho^{(t)}_i-\eta_t\gamma_1\vbar_i\right)}_{\rho'_i}e_i\\
&+
\sum_{i> i_t,i\in \cI}\underbrace{\left((1-\lambda\eta_t)\mu^{(t)}+\gamma_1\frac{\eta_t}{2n}\right)}_{\mu'}e_i\\
&+
\underbrace{\left((1-\lambda\eta_t)\mu^{(t)}+\gamma_1\frac{\eta_t}{2n}-\gamma_3\eta_t\right)}_{\xi'_{i_t}} e_{i_t}\\
&+
\sum_{i< i_t,i\in \cI}\underbrace{ \left((1-\lambda\eta_t)\xi^{(t)}_i+\gamma_1\frac{\eta_t}{2n}\right)}_{\xi'_i}e_i.\end{align*}
With this notation note that 
\[w_{t+1}=\Pi_{\W}[w_{t+1/2}]=\frac{w_{t+1/2}}{\min\{\|w_{t+1/2}\|,1\}}.\]

We next show that all three conditions are met.

For $\rho^{(t+1)}_i=\frac{\rho'_i}{\min\{\|w_{t+1/2}\|,1\}}$, note that since $0<\vbar_i<1$ as well as $0<\lambda\eta_t <1$, \[0\ge (1-\lambda \eta_t)\rho_i^{(t)}-\eta_t\gamma_1\vbar \ge -\gamma_1\sum_{i=1}^{t-1} \eta _i -\eta_t\gamma_1\ge -\gamma_1\sum_{i=1}^t \eta_t.\]
In particular $-\gamma_1\sum_{i=1}^t \eta_i\le \frac{\rho'_i}{\min\{\|w_{t+1/2}\|,1\}}\le 0 .$

Next, we examine $\mu^{(t)}=\frac{\mu'}{\min\{\|w_{t+1/2}\|,1\}}$. First, note that because $f$ is $3$-Lipschitz: \[\|w_{t+1/2}\| = \|(1-\lambda\eta_t)w_t+\lambda\eta_t\frac{1}{\lambda}\nabla F_S(w_t)\|\le \frac{3}{\lambda}.\]
Hence,
\[\mu^{(t+1)}\ge \frac{\lambda}{3}\cdot (1-\lambda \eta_t)\mu^{(t)} +\gamma_1 \frac{\eta_t}{2n})\ge \frac{\lambda}{3}\frac{\gamma_1}{2n\lambda(T+1)}\ge \frac{\gamma_1}{6n(T+1)}\ge \epsilon_d \]
That $\mu^{(t+1)}\le \gamma_1\frac{\eta_t}{2\lambda n}$, again follows by induction and the fact that $\lambda\eta_t\le 1$.
Finally, we consider $\xi^{(t+1)}$. We again use the fact that $\|w_{t+1/2}\|\le \frac{3}{\lambda}$, and we claim by induction that for $t\ge j$:
\[ \xi_{i_j}^{(t)}\le  \sum_{i=0}^{t-j}\frac{\lambda^t}{3^t}\frac{\gamma_1}{6 n}-\left(\frac{\lambda}{3}\right)^{t-j}\frac{2\gamma_3}{3(T+1)}.\]
For $j=t$ we have that $\eta^{(t)}\le \frac{\gamma_1}{2\lambda n}$ hence:
\begin{align*}
\xi_{i_t}^{(t)} &\le  
\frac{\lambda}{3}\left((1-\lambda \eta_t) \eta^{(t)}+\gamma_1\frac{\eta_t}{2n}-\gamma_3\eta_t\right) \\
&\le 
\frac{\lambda}{3}\cdot\left(\frac{\gamma_1}{2\lambda n}-\frac{2\gamma_3}{\lambda(T+1)}\right)& \lambda\eta_t\le 1\\
&\le
\frac{\gamma_1}{6n} - \frac{2\gamma_3}{3(T+1)}
\end{align*}
Next, for $t> j$ 
\begin{align*}\xi_{i_j}^{(t+1)}&\le \frac{\lambda}{3}\left[(1-\lambda\eta_t)\xi_{i_j}^{(t)} + \gamma_1\frac{\eta_t}{2n}\right]\\
&\le
\frac{\lambda}{3}\left[\xi_{i_j}^{(t)}+\frac{\gamma_1}{2\lambda n}\right]& \lambda\eta\le 1\\
&= \frac{\lambda}{3}\xi_{i_j}^{(t)}+ \frac{\gamma_1}{6n}\\
&\le \sum_{i=0}^{t+1-j}\frac{\lambda^t}{3^t} \frac{\gamma_1}{6n}-\left(\frac{\lambda}{3}\right)^{t+1-j}\frac{2\gamma_3}{3(T+1)}
\end{align*}
Finally, $\lambda <1$ and our choice of 
\[\gamma_1\le 6n\left(\frac{\lambda}{3}\right)^{T+1} \frac{4}{9}\frac{\gamma_3}{(T+1)}\]
ensures $\xi_i^{(t)}\le 0$.
\paragraph{Proof of \cref{cl:cl4}.}\label{prf:cl4}
To prove \cref{cl:cl4} we first show that for $t_0> T/2$, we have that $w'_{t_0} = (1-\lambda \eta _t)w'_{t_0} - \eta_t e_{i_{t_0}}$. In other words, there are no projections after the $T/2$'th iteration.
To show that we use the fact that $w'_t$ and $e_{i_t}$ are orthogonal for every $t$. This follows from the fact that $w'_t=\mathrm{span}(e_{i_1},\ldots, e_{i_{t-1}})$. As such,
    \begin{align*}
        \|(1-\lambda \eta _{t+1})w'_t - \eta_{t+1} \cdot \gamma_3 e_{i_t}\|^2 &= (1-\lambda \eta_{t+1})^2 \|w'_t\|^2 + (\gamma_3\eta_{t+1})^2\\
        &\le \left(1-\frac{2}{t+2}\right)^2 + \frac{4\gamma_3^2}{(\lambda(t+2))^2}\\
        & = 1 -\frac{4}{t+2} +\frac{4}{(t+2)^2} +\frac{4\gamma_3^2}{(\lambda(t+2))^2} \\
        &\le 1 -\frac{4}{t+2} +\frac{8}{T(t+2)} +\frac{8\gamma_3^2}{\lambda^2T(t+2)} & t+2\ge T/2\\ 
        & \le  1 +\left(8\frac{T/2}{T\cdot (t+1)}-\frac{4}{t+1}\right)& \gamma_3\le \frac{\lambda}{2} \sqrt{T-2}\\
        &\le 1.
    \end{align*}

Using the fact that there are no projections taken, we prove by induction that at step $t$ for $t\ge T/2$,
\[ w'_t =\frac{2}{\lambda t(t+1)}\left( \frac{\lambda T/2(T/2+1)}{2}w'_{T/2}-\sum_{k=T/2+1}^{t} \gamma_3 k e_{i_k}\right).\]
Indeed, set $t_0> T/2$, and denote $c=\frac{\lambda T/2(T/2+1)}{2}$:

    \begin{align*} w'_{t_0} & = (1-\lambda \eta_{t_0}) w'_{t_0-1} -\gamma_3\eta_{t_0} e_{i_{t_0}}\\
    & = 
    (1-\frac{2}{t_0+1})\frac{2}{\lambda t_0(t_0-1)}\left(cw'_{T/2}-\gamma_3\sum_{t=T/2+1}^{t_0-1} t e_{i_t}\right) -\gamma_3\eta_{t_0}e_{i_{t_0}}\\
    &=
    -\frac{t_0-1}{t_0+1}\cdot\frac{1}{\lambda (t_0-1)t_0}\left(c w'_{T/2}-\gamma_3\sum_{t=T/2+1}^{t_0-1} t e_{i_t}\right) -\gamma_3\frac{2}{\lambda t_0+1}e_{i_{t_0}}\\
    &= 
    \frac{2}{\lambda t_0(t_0+1)}\left(c w'_{T/2}-\gamma_3\sum_{t=T/2+1}^{t_0} te_{i_t}\right).
    \end{align*}

    \subsection{Proof of \cref{thm:overfitl}}\label{prf:overfit}
We next state a lower bound for overtraining in the general (non strongly-convex) case. As discussed, the work of \citet{bassily2020stability} showed that the stability of GD is governed by the $\eta\sqrt{T}$ term. The following result complements their work and gives a matching lower bound on the expected population risk via a construction a \citet{shalev2009stochastic} with a unique empirical risk minimizer that overfits:
\begin{theorem}\label{lem:overfit_main}
For every $\eta,T$ and $n$, for $d\geq T\cdot 2^{n+5}$, there exists a function $f(w;z):\reals^d\to\reals$ convex and $3$-Lipschitz in $w\in \cW$ for every $z$, and a distribution $D$ over $\Z$ such that if $S\sim D^n$ , then
\begin{align*}
    \E_{S\sim D^n}\brk[s]{F(\wgd_S)} - \min_{w\in\cW}F(w)
    \geq 
    \min\brk[c]2{ \max\brk[c]2{\frac{1}{8} - \frac{2^{2n+2}}{4\eta T},0}, \tfrac{1}{48}}
    .
\end{align*}
\end{theorem}
The above suggests that for large $T$ training steps, GD might overfit. In particular, when $T=\Omega\brk{2^{2n}/\eta}$, GD is susceptible to over-training.

\paragraph{Proof of \cref{lem:overfit_main}}\label{prf:overfit_main}

Define a family of convex functions $\f{\discolorlinks{\ref{eq:f_overfit}}}(w;\alpha):\reals^d \to \reals$ parameterized by $\alpha\in \brk[c]{0,1}^d$
\begin{align} \label{eq:f_overfit}
    \f{\discolorlinks{\ref{eq:f_overfit}}}(w;\alpha)
    =
    \sqrt{\sum_{i\in [d]}\alpha(i) w^2(i)} + \frac{1}{d^2}\sum_{i\in [d]}\brk{1-w(i)}
    ,
\end{align}
Observe that the objective is convex and $2$-Lipschitz. Fix $n\geq 1$ and consider the sequence $(\alpha_1,\dots,\alpha_n)$. Then, we denote the empirical average over a sample $S$ of size $n$ as follows
\begin{align} \label{eq:f_empirical_overfit}
    F_S(w)
    =
    \frac{1}{n}\sum_{i=1}^n\f{\discolorlinks{\ref{eq:f_overfit}}}(w;\alpha_i)
    .
\end{align}
Similarly to our first construction we consider $D$ to be the uniform distribution over the functions $\brk[c]{f(w;\alpha)}_{\alpha\in \brk[c]{0,1}^d}$ and we denote
\begin{align*}
    \F{\discolorlinks{\ref{eq:f_overfit}}}(w_S)
    = 
    \E_{\alpha\sim D} \brk[s]{\f{\discolorlinks{\ref{eq:f_overfit}}}(w_S;\alpha)}
    .
\end{align*}

\begin{claim} \label{lem:overfit}
    Running GD over the function $F_S(w)$ defined in \cref{eq:f_empirical_overfit} and denoting its output by $w_S$, then for $\eta \sqrt{T} \leq \frac{1}{2}$ and $d= 2^{n+1}$ the following holds
    \begin{align*}
        \E_{S\sim D^n}\brk[s]{\F{\discolorlinks{\ref{eq:f_overfit}}}(w_S)} - \min_{w\in\cW}\F{\discolorlinks{\ref{eq:f_overfit}}}(w)
        \geq
        \max\brk[c]2{\frac{1}{8} - \frac{2^{2n+2}}{4\eta T},0}
        .
    \end{align*}
\end{claim}
Note that \cref{lem:overfit} holds for $\eta \sqrt{T} \leq \tfrac{1}{2}$. Using \cref{thm:gen_main} we know that for $\eta\sqrt{T} > \tfrac12$, if $d\geq T\cdot 2^{n+5}$ there exist a function $f(w;z)$ and a distribution $D$ such that
\begin{align*}
    \E_{S\sim D^n}\E_{z\sim  D} \brk[s]{ f(w_S;z)} - \min_{w\in\cW}\E_{z\sim D}\brk[s]{ f(w;z)}
    \geq
    \tfrac{1}{16} \min\brk[c]1{\eta\sqrt{T} ,\tfrac{1}{3}}
    \geq
    \tfrac{1}{48}
    .
\end{align*}
Combining both claims for $\eta\sqrt{T}>\tfrac{1}{2}$ and $\eta\sqrt{T}\leq\tfrac{1}{2}$, we conclude the desired result. 
We now proceed with proving \cref{lem:overfit}.
\begin{proof}[of \cref{lem:overfit}]
    As $\alpha$ is distributed uniformly over $\brk[c]{0,1}^d$, we have that $\alpha(i)$ are i.i.d.~uniform Bernoulli. Consider a sample $\brk{\alpha_1,\dots,\alpha_n}$, then the probability that a given index $i$ satisfies $\forall j\in [n]:\; \alpha_j(i)=0$ is $p=\frac{1}{2^{n+1}}$. Therefore, the probability of non-existence of such coordinate is then given by $\brk{1-p}^d$. As a result, the probability that there exists such a coordinate is $1-\brk{1-\frac{1}{2^{n+1}}}^{2^{n+1}}\geq 1-e^{-1} \geq 1/2$. Recall that
    \begin{align*}
        F_S(w)
        =
        \frac{1}{n}\sum_{j=1}^n\f{\discolorlinks{\ref{eq:f_overfit}}}(w;\alpha_j)
        =
        \frac{1}{n}\sum_{j\in [n]}\sqrt{\sum_{i\in [d]}\alpha_j(i) w^2(i)} + \frac{1}{d^2}\sum_{i\in [d]}\brk{1-w(i)}
        .
    \end{align*}
    Suppose that the GD solution is in the interior of the domain, namely $\norm{w_S}<1$. In addition, suppose that there exists an index $i^\star\in[d]$ such that $\forall j\in [n]:\; \alpha_j(i^\star)=0$. We can now propose a better alternative solution denoted by $\hat w_S$ and defined as,
    \begin{align*}
        \hat w_S(i)
        =
        \begin{cases}
            w_S(i) + 1 - \norm{w_S} & i=i^\star;\\
            w_S(i) & i\neq i^\star.\\
        \end{cases}
    \end{align*}
    Observe that
    \begin{align*}
        F_S(w_S) - F_S(\hat w_S) 
        \geq 
        \frac{1}{d^2}(1-\norm{w_S})
        .
    \end{align*}
    Using the well known optimization upper bound of GD on convex $2$-Lipschitz functions \cite[Theorem~3.2]{Bubeck15} we obtain for $\eta\sqrt{T}\leq \tfrac12$:
    \begin{align*}
        F_S(w_S) - F_S(\hat w_S)
        \leq
        \frac{1}{2\eta T}+ 2\eta
        \leq
        \frac{1}{\eta T}
        ,
    \end{align*}
    where we used the fact that our domain is bounded in the Euclidean unit ball. This implies that
    \begin{align} \label{eq:norm_overfit}
        \norm{w_S}
        &\geq
        1 - \frac{d^2}{\eta T}
        .
    \end{align}
    Consequently, we can conclude that with probability higher than $1/2$ we get
    \begin{align*}
        \F{\discolorlinks{\ref{eq:f_overfit}}}(w_S) - \min_{w\in\cW}\F{\discolorlinks{\ref{eq:f_overfit}}}(w)
        &\geq
        \F{\discolorlinks{\ref{eq:f_overfit}}}(w_S) - \tfrac14 \\
        &\geq
        \E_{\alpha\sim D} \brk[s]2{\sqrt{\sum_{i\in [d]}\alpha(i) w^2_S(i)}} - \tfrac14 \tag{$1-w(i)\geq 0$ for $w\in\cW$}\\
        &\geq
        \tfrac12\norm{w_S} - \tfrac14 \tag{Jensen's inequality with $\E\brk[s]{\alpha(i)}=\tfrac12$}\\
        &\geq
        \frac14 -\frac{d^2}{2\eta T} \tag{\cref{eq:norm_overfit}}
    \end{align*}
    where the first inequality stems from the observation that $\f{\discolorlinks{\ref{eq:f_overfit}}}(0;\alpha)=1/d=2^{-\brk{n+1}}\le \frac{1}{4}$. For the third inequality, a simple observation of $\alpha(i)=\alpha^2(i)$ ensures that the first term is convex. Note that when $\norm{w_S}=1$ we get an even tighter lower bound of $1/4$. After taking expectation over $S\sim D^n$ we conclude that
    \begin{align*}
        \E_{S\sim D^n}[\F{\discolorlinks{\ref{eq:f_overfit}}}(w_S)] - \min_{w\in\cW}\F{\discolorlinks{\ref{eq:f_overfit}}}(w)
        &\geq
        \max\brk[c]2{\frac{1}{8} - \frac{2^{2n+2}}{4\eta T},0}
        .
        \qedhere
    \end{align*}
\end{proof}

\subsection*{Acknowledgments} 

The authors would like to thank Assaf Dauber, Vitaly Feldman and Kunal Talwar for helpful discussions. 
This work was partially supported by the Israeli Science Foundation (ISF) grants 2549/19 and 2188/20, by the Len Blavatnik and the Blavatnik Family foundation, and by the Yandex Initiative in Machine Learning. 

\bibliographystyle{abbrvnat}
\bibliography{refs}

\appendix

\section{Additional proofs for \cref{thm:main} and \cref{thm:strongmain}}

\subsection{Proof of \cref{lem:opt_1_main}}\label{prf:opt_1_main}

Without loss of generality assume $18 \eta^2T^2\le d \le 36 \eta^2T^2$ (as we can always embed the below example in any larger space). Set parameters $0 < \epsilon_1 < \dots < \epsilon_d<\frac{1}{2\sqrt{d}}$, and
define the deterministic convex function $\f{\discolorlinks{\ref{eq:f_opt_1}}} : \R^{d} \to \R$ as follows:
\begin{align}\label{eq:f_opt_1}
    \f{\discolorlinks{\ref{eq:f_opt_1}}}(w)
    =
    \norm1{w-\frac{1}{\sqrt{d}} + \epsilon}_\infty
    ,
\end{align}
where $\epsilon=\brk{\epsilon_1,\dots,\epsilon_d}$.
We begin the proof with the following claim that upper bounds the smallest coordinate:
    \begin{claim} \label{cla:opt}
        There exists an $i\in[d]$ such that
        $
            w_S(i)
            \leq
            \frac{\eta T}{d}
            .
        $
    \end{claim}
    \begin{proof}
        The update rule of GD states that
        \begin{align*}
            w_{t}
            =
            \Pi_\cW \brk[s]{w_{t-1}-\eta \nabla f(w_{t-1})}
            .
        \end{align*}
        Note that 
        \begin{align}\label{eq:f_opt_gradient}
            \nabla f(w)
            =
            -\sign\brk{\frac{1}{\sqrt{d}} - w(i) - \epsilon_{i}}e_{i},
        \end{align} 
        for an $i\in \arg\max_{j\in[d]}\abs{\frac{1}{\sqrt{d}}-w(j)-\epsilon_j}$. We will upper bound the $\ell_1$-norm of $w_{t+1}$. Observe that projection can only reduce the $\ell_1$-norm. Therefore,
        \begin{align*}
            \norm{w_{t}}_1 
            &\leq 
            \norm{w_{t-1}-\eta \nabla f(w_t)}_1 \\
            &\leq 
            \norm{w_{t-1}}_1+\eta\norm{\nabla f(w_t)}_1 \tag{\textrm{triangle inequality}} \\
            &\leq
            \norm{w_{t-1}}_1+\eta \tag{\textrm{\cref{eq:f_opt_gradient}}} \\
            &\leq
            \eta t \tag{\textrm{applying the claim recursively on $\norm{w_{t-1}}_1$}}
            .
        \end{align*}
        This implies that the average iterate holds 
        \begin{align} \label{eq:ws_norm}
            \norm{w_S}_1 
            \leq
            \frac{1}{T}\sum_{t=1}^T\norm{w_t}_1
            \leq
            \frac{1}{T}\sum_{t=1}^T\eta t
            \leq
            \eta T
            .
        \end{align}
        If we assume by contradiction that for all $i\in[d]$ it holds that $w_S(i)>\eta T / d$, then we will get that $\norm{w_S}_1>\eta T$ which contradicts the claim in \cref{eq:ws_norm}.
    \end{proof}
    Using \cref{cla:opt} the average iterate satisfies $w_S(i) \leq \frac{\eta T}{d}$ for some $i$ and we can conclude
    \begin{align*}
        f(w_S)
        \geq
        \frac{1}{\sqrt{d}}-\frac{\eta T}{d}-\epsilon_i
        \geq
        \frac{1}{2\sqrt{d}}-\frac{\eta T}{d}
        ,
    \end{align*}
    since $\epsilon_i \leq \frac{1}{2\sqrt{d}}$. 
    We obtain that for $18\eta^2T^2 \geq 1$ 
    \begin{align}
        f(w_S) 
        \geq 
        \frac{1}{12\eta T} - \frac{1}{18\eta T}
        =
        \frac{1}{36\eta T}
        ,
    \end{align}
    where we used the fact that $18\eta^2T^2 \leq  36\eta^2T^2$. While for $18\eta^2 T^2 \leq 1$ we get that for $d=1$
    \begin{align*}
        f(w_S) 
        \geq
        \frac12 - \eta T
        \geq
        \frac14
        .
    \end{align*}
Note also that $f\brk{\sum_{i\in [d]} \brk{\frac{1}{\sqrt{d}} - \epsilon_i}e_i}=0$,
 hence:
 \[F(w_S)-\min_{w\in \W}F(w) \ge \min\left\{\frac{1}{36 \eta T},\frac{1}{4}\right\}
 \ignore{ \ge \min\left\{\frac{1}{144 \eta T}+\frac{1}{2\sqrt{n}},\frac{1}{4}\right\}=\frac{1}{36}\min\left\{\frac{1}{4\eta T}+\frac{18}{\sqrt{n}},9\right\}}.\]
Because we consider a deterministic function, the result also holds in expectation.

\subsection{Proof of \cref{cl:lambda}}\label{prf:lambda}
Consider the function
\[ h(w) = - \frac{\bar\lambda}{2}w(1),\]
where $\bar\lambda=\min\{1,\lambda\}$

We set the distribution $D$ to be deterministic, namely: $f(w,z)=h(w)=-\bar{\lambda}w(1)$ w.p. $1$. 

 Note that for our update step we have that \[w_1=-\frac{\bar{\lambda}}{2\lambda}\nabla h(w) =-\alpha \cdot e_1,\]
 where $\alpha\le \frac{1}{2}$
 Since $\nabla F_{\lambda,S}(w_1)=0$,we have that for every $t\ge 1$, $w_t=w_1$, in particular, $w_S=-e_1$. On the other hand:
\[h(\frac{\alpha\cdot\bar\lambda}{2} e_1)-h(e_1)\ge -\frac{\alpha\bar{\lambda}}{2}+\frac{\bar{\lambda}}{2} =\frac{\bar\lambda}{4}\ge \min \left\{\frac{\lambda}{4},\frac{1}{4}\right\}\ignore{\ge \min \frac{1}{8}\left\{\lambda+\min\left\{\frac{4}{\lambda n},\frac{1}{\sqrt{n}}\right\},2\right\}}.\] 
\section{Proof of \cref{lem:gen}}\label{prf:gen}

Because we only analyze non-regularized objectives in this proof we will suppress the dependence on the algorithm and we will use $w_t$ for $\wgd_t$. Before we proceed with the proof, we present a generic \namecref{lem:core} that this proof relies upon. The proof of \cref{lem:core} can be found at \cref{prf:core}.
\begin{lemma}\label{lem:core}
Fix $\epsilon\in \reals^d$ and $\gamma\in \reals^2$. Let $(z_1,\ldots z_n)$ be a sequence such that $z_j=(\alpha_j,\epsilon,\gamma)$, and consider
\[ 
    F(w)
    = 
    \frac{1}{n}\sum_{j\in [n]} \f{\discolorlinks{\ref{eq:f_gen}}}(w;z_j)
    .
\]
Denote $\nabla_iF(w)$ to be the $i$-th element of the gradient $\nabla F(w)$.
If $\;\cI = \{i : \forall j\in[n],\; \alpha_j(i)=0\}$, then for a choice $0<\epsilon_1 < \ldots < \epsilon_d$ we have the following:
\begin{enumerate}
    \item $\nabla F(0)=\gamma_1\vbar$.
    \item For every $i \in \cI$ then $\nabla_i F(0)= -\frac{\gamma_1}{2n}$.
    \item If $i\notin \cI$ then $0 < \nabla_i F(0) \le \gamma_1$.
    \item Suppose for some $k\in \cI$:
    \[
        w
        =
        \sum_{i\notin \cI} \rho_i e_i + \sum_{i\ge k,i\in \cI} \mu e_i+ \sum_{i < k, i\in \cI}\xi_i e_i
        ,
    \]
    where $-\gamma_2 < \rho_i< 0$, $\xi_i\le 0$ and, $\mu > \epsilon_d$. Then:
    \[
        \nabla F(w)
        =
        \gamma_1\vbar + \gamma_3e_k
        ,
    \]
    where $\vbar = \frac{1}{n}\sum_{j \in [n]} v_{\alpha_j}$ and $e_k$ is the $k$-th standard basis vector in $\reals^d$.
    \item Suppose:
    \[
        w
        =
        \sum_{i\notin \cI} \rho_i e_i + \sum_{i\in \cI}\xi_i e_i
        ,
    \]
    where $-\gamma_2 < \rho_i< 0$ and $\xi_i\le 0$. Then:
        \[
        \nabla F(w)
        =
        \gamma_1\vbar
        .
    \]
\end{enumerate}
\end{lemma}

First, we address the case of $t\leq \min\brk[c]{T,K}$. We will prove by induction on $t$ that 
\begin{align} \label{eq:wgd_1}
    w_t
    =
    -\eta t \cdot \gamma_1 \vbar - \eta\sum_{s=1}^{t-1} e_{i_s}
    .
\end{align}
For $t=1$ we know that the first GD step takes to $w_1=-\eta\nabla F_S(0)$. Using \cref{lem:core} we get that $w_1=-\eta \gamma_1 \vbar$ which concludes the base of the induction. For the induction step we assume that $w_t$ is given by \cref{eq:wgd_1}. Recall that $0<\vbar_i\leq 1$ for $i\notin \cI$ and $\vbar_i=-\frac{1}{2n}$ for $i\in\cI$. Now observe that $w_t$ takes the following form
\begin{align*}
    w_t
    &=
    \sum_{i\notin\cI} \rho^{(t)}_i e_i + \sum_{i\in \cI}\frac{\gamma_1}{2n}\eta t e_i + \sum_{i \leq i_{t-1}, i \in \cI}\brk{-\eta} e_i \\
    &=
    \sum_{i\notin\cI} \rho^{(t)}_i e_i + \sum_{i \geq i_t, i \in \cI}\frac{\gamma_1}{2n}\eta t e_i + \sum_{i < i_t, i \in \cI}\brk{-\eta + \frac{\gamma_1}{2n}\eta t} e_i
    ,
\end{align*}
where $-\gamma_1\eta t \leq \rho^{(t)}_i < 0$. One can ensure that all the necessary conditions of the fourth claim in \cref{lem:core} hold, under the assumptions of \cref{lem:gen}. Namely, that $-\gamma_2<\rho^{(t)}_i<0$, $\brk1{-\eta + \frac{\gamma_1}{2n}\eta t} \leq 0$ and $\frac{\gamma_1}{2n}\eta t > \epsilon_d$. Therefore, we can apply \cref{lem:core} and obtain that $\nabla F_S(w_t)=\gamma_1\vbar + e_{i_t}$. The next iterate is then 
\begin{align*}
    w_{t+1}
    &=
    w_t - \eta \gamma_1\vbar - \eta e_{i_t} \\
    &=
    -\eta (t+1) \cdot \gamma_1 \vbar - \eta\sum_{s=1}^t e_{i_s}
    ,
\end{align*}
which concludes the first part of the proof. We now address the case of $K < t \leq T$ when $T>K$. Similarly to the first part, we will prove by induction on $t$ that
\begin{align} \label{eq:wgd_2}
    w_t
    =
    -\eta t \cdot \gamma_1 \vbar - \eta\sum_{s=1}^K e_{i_s}
    .
\end{align}
Starting at $t=K+1$ we know that $w_{K+1}=w_K-\eta\nabla F_S(w_K)$. From the first part of the claim we can deduce that $w_K$ holds \cref{eq:wgd_1} and $\nabla F_S(w_K)=\gamma_1\vbar + e_{i_K}$. Therefore we conclude that
\begin{align*}
    w_{K+1}
    =
    w_K - \eta \gamma_1 \vbar - \eta e_{i_K}
    =
    -\eta (K+1) \cdot \gamma_1 \vbar - \eta\sum_{s=1}^K e_{i_s}
    .
\end{align*}
For the induction step we assume that $w_t$ holds \cref{eq:wgd_2}. Taking advantage of the properties of $\vbar$ we have that $w_t$ takes the following form
\begin{align*}
    w_t
    &=
    \sum_{i\notin\cI} \rho^{(t)}_i e_i + \sum_{i\in \cI}\frac{\gamma_1}{2n}\eta t e_i + \sum_{i \in \cI}\brk{-\eta} e_i \\
    &=
    \sum_{i\notin\cI} \rho^{(t)}_i e_i + \sum_{i \in \cI}\brk{-\eta + \frac{\gamma_1}{2n}\eta t} e_i
    ,
\end{align*}
where $-\gamma_1\eta t \leq \rho^{(t)}_i < 0$. Again, one can ensure that all the necessary conditions of the fifth claim in \cref{lem:core} hold, under the assumptions of \cref{lem:gen}. Namely, that $-\gamma_2<\rho^{(t)}_i<0$ and $\brk1{-\eta + \frac{\gamma_1}{2n}\eta t} \leq 0$. Applying \cref{lem:core} we obtain that $\nabla F_S(w_t)=\gamma_1\vbar$. We then conclude that
\begin{align*}
    w_{t+1}
    &=
    w_t - \eta \gamma_1\vbar \\
    &=
    -\eta (t+1)\cdot \gamma_1 \vbar - \eta \sum_{s=1}^K e_{i_s}
    .
\end{align*}

Because we ignored projections throughout the proof we need to ensure that each $w_t$ for $t=0,\ldots,T$ lies in the Euclidean unit ball. Observe that this is indeed the case, as we get
\begin{align*}
    \norm{w_t}^2
    \leq
    \eta^2 \min\brk[c]{t, K} + d \gamma_1^2\eta^2 T^2
    \leq
    \eta^2 \brk{K + d \gamma_1^2 T^2}
    \leq
    1
    ,
\end{align*}
since $\gamma_1 \leq \frac{1}{2\sqrt{d}\eta T}$ and $K\leq \frac{3}{4\eta^2}$.

\subsection{Proof of \cref{lem:core}} \label{prf:core}
Consider the gradient of $\f{\discolorlinks{\ref{eq:f_gen}}}(w;z_j)$ at $w=0$,
\begin{align*}
    \nabla\f{\discolorlinks{\ref{eq:f_gen}}}(0;z_j)
    =
    \nabla \brk3{\sqrt{\sum_{i\in [d]}\alpha_j(i) h_\gamma^2(w(i))}}\Bigg\rvert_{w=0} + \gamma_1v_{\alpha_j} + \gamma_3\nabla r_\epsilon(0)
    .
\end{align*}
Observe that $\nabla r_\epsilon(w)=0$ for any $w$ that satisfies 
$
    \forall i\in [d]:
    \;
    w(i)
    < 
    \epsilon_1
    .
$
In particular, this holds when $w=0$. Now, note that for any $w(i) > -\gamma_2$ we have $h_\gamma(w(i))=0$. Since $\gamma_2 > 0$, this implies that
\begin{align*}
    \forall j \in [n]:
    \quad
    \nabla \brk3{\sqrt{\sum_{i\in [d]}\alpha_j(i) h_\gamma^2(w(i))}}\Bigg\rvert_{w=0}
    =
    0
    ,
\end{align*}
and we obtain 
$
    \nabla\f{\discolorlinks{\ref{eq:f_gen}}}(0;z_j)
    =
    \gamma_1v_{\alpha_j}
    .
$
This concludes the first claim proof, as we get
\begin{align*}
    \nabla F(0)
    =
    \frac{\gamma_1}{n}\sum_{j\in [n]}v_{\alpha_j}
    =
    \gamma_1\vbar
    .
\end{align*}
For $i \in \cI$, note that $v_{\alpha_j}(i) = -\frac{1}{2n}$ and therefore $\vbar_i=-\frac{1}{2n}$. In addition, for $i\notin \cI$ there is at least one sample $j\in [n]$ such that $\alpha_j(i)=1$. This entails that for any $i \notin \cI$ it holds $0 < \vbar_i=\frac{1}{n}\sum_{j\in [n]}v_{\alpha_j}(i) \leq 1$. Taking both cases conclude the second and third claim proofs. For the fourth claim we assume that for some $k\in \cI$
\begin{align*}
    w 
    = 
    \sum_{i\notin \cI} \rho_i e_i+ \sum_{i\ge k,i\in \cI}\mu e_i+ \sum_{i < k, i\in \cI}\xi_i e_i
    .
\end{align*}
Consider then the gradient at $w$
\begin{align*}
    \nabla F(w)
    =
    \frac{1}{n}\sum_{j\in [n]}\nabla \brk3{\sqrt{\sum_{i\in [d]}\alpha_j(i) h_\gamma^2(w(i))}} +\frac{\gamma_1}{n}\sum_{j\in [n]} v_{\alpha_j}(i) + \gamma_3\nabla r_\epsilon(w)
    .
\end{align*}
Let us examine each term separately. We start with the first term,
\begin{align*}
    \frac{1}{n}\sum_{j\in [n]}\nabla \brk3{\sqrt{\sum_{i\in [d]}\alpha_j(i) h_\gamma^2(w(i))}}
    &=
    \frac{1}{n}\sum_{j\in [n]}\nabla \brk3{\sqrt{\sum_{i\notin \cI}\alpha_j(i) h_\gamma^2(w(i))}}
    ,
\end{align*}
where we used the fact that for any $i\in \cI$ we have $\alpha_j(i)=0$. First, note that this term is independent in $w(i)$ for $i\in \cI$. In addition, for any $i\notin \cI$ we have $w(i)=\rho_i > -\gamma_2$. This implies that
\begin{align} \label{eq:empirical_gradient}
    \frac{1}{n}\sum_{j\in [n]}\nabla \brk3{\sqrt{\sum_{i\in [d]}\alpha_j(i) h_\gamma^2(w(i))}}
    =
    0
    .
\end{align}
The second term is trivially given by the definition of $\vbar$,
\begin{align*}
    \frac{\gamma_1}{n}\sum_{j\in [n]} v_{\alpha_j}
    =
    \gamma_1 \vbar
    .
\end{align*}
Recall that $r_\epsilon(w)=\max\brk[c]1{0,\max_{i \in [d]}\brk[c]{w(i) - \epsilon_i}}$ and observe the following
\begin{align*}
    w(i)
    =
    \begin{cases}
        \rho_i & i\notin \cI; \\
        \mu & i\in \cI, i \geq k; \\
        \xi_i & i\in \cI, i < k. \\
    \end{cases}
\end{align*}
Since $\rho_i<0$ and $\xi_i \leq 0$ the maximum of $w_i-\epsilon_i$ can only be achieved for $i\in \cI, i\geq k$. Specifically, as $\epsilon_i$ are strictly increasing and $\mu > \epsilon_i$ for all $i\in[d]$, we get that the maximum is given in $i=k$. This concludes the fourth claim proof as $\nabla r_\epsilon(w)=e_k$. For the last claim we assume 
\begin{align*}
    w 
    = 
    \sum_{i\notin \cI} \rho_i e_i + \sum_{i\in \cI}\xi_i e_i
    .
\end{align*}
Following the same arguments as in the previous claim we get that \cref{eq:empirical_gradient} holds here as well. Therefore,
\begin{align*}
    \nabla F(w) 
    = 
    \gamma_1\vbar + \gamma_3\nabla r_\epsilon(w)
    .
\end{align*}
Observe that $w(i)\leq 0$ for any $i\in[d]$, since $\rho_i < 0$ and $\xi_i\leq 0$. This implies that $\nabla r_\epsilon(w)=0$, which concludes the proof.

\end{document}